\documentclass{article}

\usepackage{microtype}
\usepackage{graphicx}
\usepackage{subfigure}
\usepackage{booktabs} 
\usepackage{enumitem}
\setlist[itemize]{leftmargin=*}
\usepackage{pifont}

\usepackage{listings}
\usepackage[dvipsnames]{xcolor}
\usepackage[table]{xcolor}  

\usepackage{hyperref}
\usepackage[preprint]{icml2026}

\usepackage{amsmath}
\usepackage{amssymb}
\usepackage{mathtools}
\usepackage{amsthm}
\usepackage[table]{xcolor}

\usepackage[capitalize,noabbrev]{cleveref}

\theoremstyle{plain}
\newtheorem{theorem}{Theorem}[section]
\newtheorem{proposition}[theorem]{Proposition}
\newtheorem{lemma}[theorem]{Lemma}
\newtheorem{corollary}[theorem]{Corollary}
\theoremstyle{definition}

\newtheorem{assumption}[theorem]{Assumption}
\theoremstyle{remark}

\usepackage[textsize=tiny]{todonotes}

\icmltitlerunning{Efficient Cross-Domain Offline Reinforcement Learning with Dynamics- and Value-Aligned Data Filtering}

\begin{document}

\twocolumn[
  \icmltitle{Efficient Cross-Domain Offline Reinforcement Learning with Dynamics- and Value-Aligned Data Filtering}



  \icmlsetsymbol{equal}{*}

  \begin{icmlauthorlist}
    \icmlauthor{Zhongjian Qiao}{cityu}
    \icmlauthor{Rui Yang}{UIUC}
    \icmlauthor{Jiafei Lyu}{Tencent}
    \icmlauthor{Chenjia Bai}{Tele}
    \icmlauthor{Xiu Li}{thu}
    \icmlauthor{Siyang Gao}{cityu}
    \icmlauthor{Shuang Qiu${}^\dag$}{cityu}
  \end{icmlauthorlist}

  \icmlaffiliation{cityu}{City University of Hong Kong}
  \icmlaffiliation{UIUC}{University of Illinois Urbana-Champaign}
  \icmlaffiliation{Tencent}{Tencent}
  \icmlaffiliation{Tele}{Institute of Artificial
Intelligence, China Telecom (TeleAI)}
  \icmlaffiliation{thu}{Tsinghua University. ${}^\dag$Corresponding Author}

  \icmlcorrespondingauthor{Shuang Qiu}{shuanqiu@cityu.edu.hk}

  \icmlkeywords{Machine Learning, ICML}

  \vskip 0.3in
]



\printAffiliationsAndNotice{}  

\begin{abstract}
Cross-domain offline reinforcement learning (RL) aims to train a well-performing agent in the target environment, leveraging both a limited target domain dataset and a source domain dataset with (possibly) sufficient data coverage. Due to the underlying dynamics misalignment between source and target domains, naively merging the two datasets may incur inferior performance. Recent advances address this issue by selectively leveraging source domain samples whose dynamics align well with the target domain. However, our work demonstrates that dynamics alignment alone is insufficient, by examining the limitations of prior frameworks and deriving a new target domain sub-optimality bound for the policy learned on the source domain. More importantly, our theory underscores an additional need for \textit{value alignment}, i.e., selecting high-quality, high-value samples from the source domain, a critical dimension overlooked by existing works. Motivated by such theoretical insight, we propose \textbf{\underline{D}}ynamics- and \textbf{\underline{V}}alue-aligned \textbf{\underline{D}}ata \textbf{\underline{F}}iltering (DVDF) method, a novel unified cross-domain RL framework that selectively incorporates source domain samples exhibiting strong alignment in \textit{both dynamics and values}.
We empirically study a range of dynamics shift scenarios, including kinematic and morphology shifts, and evaluate DVDF on various tasks and datasets, even in the challenging setting where the target domain dataset contains an extremely limited amount of data. 
Extensive experiments demonstrate that DVDF consistently outperforms strong baselines with significant improvements.
Our code is available at \url{https://github.com/zq2r/DVDF.git}.
\end{abstract}

\section{Introduction}
\label{sec:intro}

Reinforcement learning (RL)~\cite{sutton1999reinforcement} has seen remarkable progress in fields like video games~\cite{ye2020mastering,mnih2013playing} and robotics~\cite{kober2013reinforcement,kormushev2013reinforcement}. However, the frequent interactions required for online RL can be expensive or risky in real-world applications. Offline RL~\cite{levine2020offline,prudencio2023survey} addresses this by learning from pre-collected datasets, eliminating the need for online interaction. Yet, its performance is often constrained by the limited size of target datasets, as extensive data collection remains costly. To overcome this, cross-domain offline RL~\citep{wen2024contrastive,liu2022dara,liu2024beyond} leverages additional data (called source domain) collected from environments related to but distinct from the target one for policy learning. 

Although the idea of leveraging additional source domain data to benefit target policy learning is promising, the key challenge lies in that the source and target environments may differ in transition dynamics, and simply merging the source and target data for training could degrade the performance~\citep{wen2024contrastive, liu2024beyond} due to the out-of-distribution (OOD) transition dynamics issue~\citep{liu2024beyond}. Previous solutions for this issue include training a domain classifier for reward augmentation~\citep{liu2022dara,eysenbach2020off}, using supported value optimization and conservative regularization to mitigate overestimation~\citep{liu2024beyond}, etc. Recent advances~\citep{xu2024cross, wen2024contrastive, lyu2025cross} introduce dynamics-aware data filtering, where source domain samples are selectively shared based on their alignment with the target dynamics to enhance policy learning. For example, and IGDF~\citep{wen2024contrastive} leverages contrastive representation for data filtering, OTDF~\citep{lyu2025cross} selects source domain data based on optimal transport. Despite methodological differences, these studies share a common idea: \textit{source domain samples with smaller dynamics misalignment facilitate target policy learning, whereas those with larger misalignment impede it}. However, we argue that this point may not universally hold, as it overlooks the significance of 
\textit{value alignment}: the selected data should also exhibit high quality other than aligned dynamics. Intuitively, high-quality source samples with moderate dynamics misalignment may contribute more to target policy learning than low-quality samples that are well aligned in dynamics.
Consider the case where the source domain dataset consists of non-expert low-quality samples with minor dynamics misalignment and expert samples with larger dynamics misalignment. Methods based on dynamics-aware data filtering will only select low-quality samples, although these samples may contribute little to policy learning. 
Instead, incorporating expert samples (despite larger dynamics misalignment) may yield better performance. Therefore, we raise the question: \textit{Can we devise a cross-domain offline RL method that jointly considers dynamics alignment and value alignment?}

In this paper, we propose a simple yet effective solution for the above question, called \textbf{\underline{D}}ynamics- and \textbf{\underline{V}}alue-aligned \textbf{\underline{D}}ata \textbf{\underline{F}}iltering (DVDF). We start with a motivating example to empirically show that only considering dynamics alignment is not enough for efficient cross-domain offline RL. From a theoretical perspective, we reveal that existing theoretical frameworks that \textit{focus on tightening the performance discrepancy of a given policy between the source and target domain} \textbf{misalign with the learning objective}, and fails to guarantee learning a well-performing target policy. This explains the limitations of the recent methods like IGDF and OTDF. Alternatively, we derive a concrete sub-optimality bound for policies trained on the source domain and evaluated on the target domain, demonstrating that both dynamics and value alignments are essential for cross-domain offline RL.
Based on this theoretical insight, we present our method, DVDF, which utilizes an advantage function pre-trained on the source domain to measure the value misalignment and incorporates dynamics-aware data filtering
to capture the dynamics misalignment within a unified framework. 
Then DVDF trades off dynamics and value misalignment and selectively shares source domain samples to train the policy. DVDF can be generally treated as a \textbf{plug-in module} and seamlessly integrated with recent methods like IGDF and OTDF. 
Our contributions can be summarized as follows.

\begin{itemize}[topsep=0pt, partopsep=0pt]
    \item We examine the limitations of the current theoretical analysis framework for cross-domain offline RL, and theoretically demonstrate that both dynamics alignment and value alignment are essential for cross-domain offline RL, providing new insights for the field.
    
    \item Based on the theoretical insight, we propose our method, DVDF, which jointly considers dynamics alignment and value alignment, and selectively shares source domain data for policy learning. DVDF is a plug-in module and can be integrated into other methods like IGDF and OTDF.
    
    \item 
    We conduct extensive experiments across various dynamics shift conditions, which demonstrate that DVDF exhibits superior performance on many tasks and datasets compared to strong baselines. We further test DVDF under challenging conditions where the target domain dataset is extremely limited~\citep{lyu2024odrlabenchmark, lyu2025cross}, and observe DVDF delivers exceptional performance.
\end{itemize}

\section{Preliminaries}

We consider a Markov Decision Process (MDP)~\citep{puterman1990markov} defined by the six-tuple $\mathcal{M}=(\mathcal{S},\mathcal{A},P,r,\rho,\gamma)$ where $\mathcal{S}$ is the state space, $\mathcal{A}$ is the action space, $P: \mathcal{S}\times\mathcal{A}\rightarrow\Delta(\mathcal{S})$ is the transition dynamics, $\Delta(\cdot)$ is the probability simplex, $r(s,a):\mathcal{S}\times\mathcal{A}\rightarrow[-r_{\rm max},r_{\rm max}]$ is the reward function, $\rho$ is the initial state distribution, and $\gamma$ is the discount factor. RL aims to learn a policy $\pi:\mathcal{S}\rightarrow\Delta(\mathcal{A})$ that maximizes the objective $J_{\mathcal{M}}(\pi)\coloneqq\mathbb{E}_{\pi}\left[\sum_{t=0}^\infty\gamma^tr(s_t,a_t)\right]$. 

In the cross-domain RL setting, we assume that we have access to a \textit{source domain} $\mathcal{M}_{\text{src}}=(\mathcal{S},\mathcal{A},P_{\text{src}},r,\rho,\gamma)$ and a \textit{target domain} $\mathcal{M}_{\text{tar}}=(\mathcal{S},\mathcal{A},P_{\text{tar}},r,\rho,\gamma)$. The only difference between the two domains is the transition dynamics. In the offline setting, only a target domain dataset $\mathcal{D}_\text{tar}=\{(s_i,a_i,r_i,s_{i+1})\}_{i=1}^{N_1}$ and a source domain dataset $\mathcal{D}_\text{src}=\{(s_i,a_i,r_i,s_{i+1})\}_{i=1}^{N_2}$ are available, where $N_1\ll N_2$. The goal of cross-domain offline RL is to leverage $\mathcal{D}_\text{tar}$ and $\mathcal{D}_\text{src}$ to improve the performance of the agent in the target domain, where $\mathcal{D}_{\text{src}}$ and $\mathcal{D}_{\text{tar}}$ denote the datasets collected in the source and target domain, respectively.

\begin{figure*}[t]
    \centering
    \includegraphics[width=0.98\linewidth]{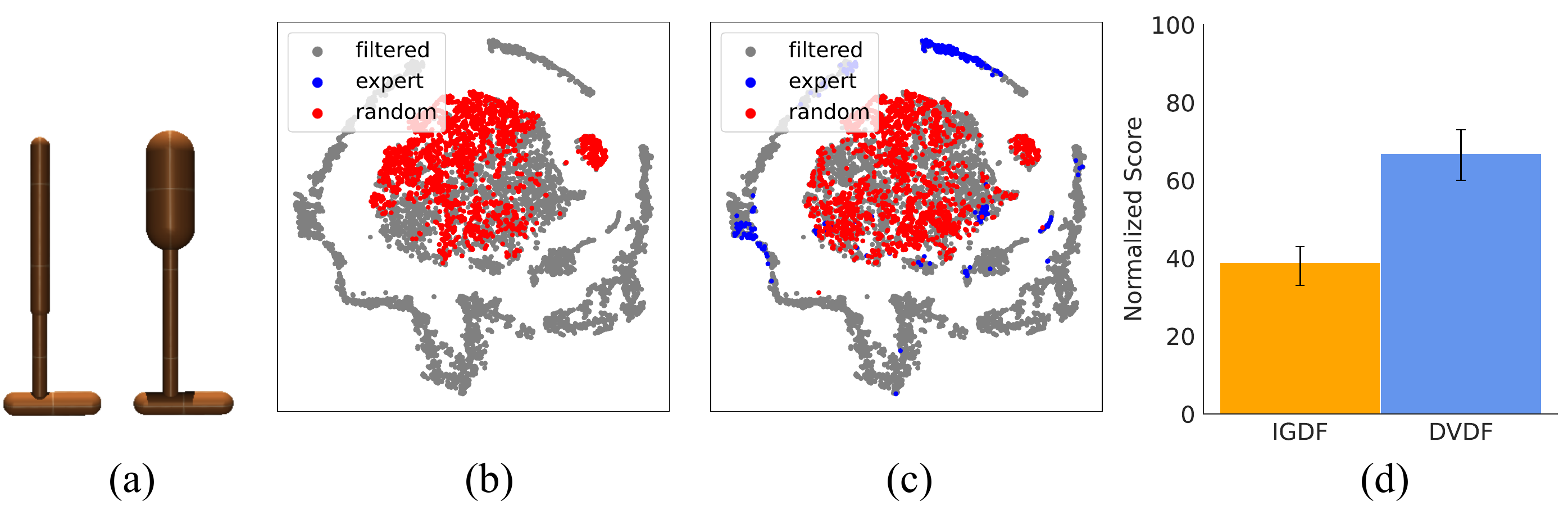}
    \caption{\textbf{(a)} Robot morphology visualization of target domain (left) and source domain (right). \textbf{(b)} Source data filtering visualization of IGDF. \textbf{(c)} Source data filtering visualization of DVDF. \textbf{(d)} Performance comparison between IGDF and DVDF on the target domain.} 
    \vspace{-0.cm}
    \label{fig:motivation}
\end{figure*}

\section{Motivating Example: Dynamics Alignment Alone Is Insufficient}
\label{sec:motivation}
In this section, we use a simple example to demonstrate our claim: solely considering dynamics alignment is insufficient for efficient cross-domain offline RL.

We consider the following cross-domain RL scenario: the target domain is \texttt{hopper-v2} task from MuJoCo~\citep{todorov2012mujoco}, and the source domain is \texttt{hopper-v2} task with morphology shift (the head size of the robot is increased), called \texttt{hopper-morph-v2}. We visualize the morphology of the robots in Figure~\ref{fig:motivation} (a). In the offline setting, we require both source and target domain datasets. For the target domain, we extract a $10\%$ subset from the \texttt{hopper-medium-v2} dataset in D4RL~\citep{fu2020d4rl}. The source domain comprises a mixture of: (1) 0.5M random-level samples from \texttt{hopper-random-v2} dataset, and (2) 0.5M expert-level samples collected by a well-trained SAC~\citep{haarnoja2018soft} policy in the \texttt{hopper-morph-v2} environment. Given such source and target domain datasets, we implement the original IGDF~\citep{wen2024contrastive} and our proposed DVDF method (based on IGDF) for source data filtering and target policy learning. We set the source data selection ratio to $25\%$ for both DVDF and IGDF.

We visualize the source data filtering results of IGDF and DVDF using t-SNE~\citep{van2008visualizing}, as shown in Figure~\ref{fig:motivation} 
 (b) and (c), respectively. The \textcolor{gray}{gray} points represent filtered samples, \textcolor{blue}{blue} points indicate selected expert samples exhibiting dynamics shifts, and \textcolor{red}{red} points denote selected random samples without dynamics shifts. The result reveals that IGDF exclusively selects random samples, whereas DVDF incorporates both random and expert samples. We further evaluate the policies trained by each method in the target environment, with the normalized score presented in Figure~\ref{fig:motivation} (d). The results demonstrate that DVDF achieves a significantly higher average score of \textbf{67} compared to 39 obtained by IGDF, representing a $\textbf{71\%}$ performance improvement. This substantial improvement demonstrates that the shifted expert data can significantly enhance policy learning, validating our motivation that efficient cross-domain policy learning requires joint consideration of both dynamics and value alignment. More details and results can be found in Appendix~\ref{appendix:motivation}.

\section{What is Truly Essential for Efficient Cross-Domain Offline RL?}

In this section, we provide theoretical insights for our motivation by rethinking and examining the limitations of the current theoretical framework for cross-domain RL. Our analysis reveals a fundamental gap in the existing theoretical foundation, prompting us to answer an important question: \textit{what is truly essential for efficient cross-domain offline RL?}

To answer the above question, we first present the key theoretical framework for recent cross-domain RL methods~\citep{wen2024contrastive, lyu2025cross, xu2024cross, lyu2024cross} in Lemma~\ref{lemma:1}, which mainly relies on establishing a performance difference bound of a given policy between the source and target domain:

\begin{lemma}[\textbf{Performance difference bounded by the dynamics misalignment}]
\label{lemma:1}
Denote the MDP of the source domain and target domain as $\mathcal{M}_{\mathrm{src}}$ and $\mathcal{M}_{\mathrm{tar}}$. We have the performance difference of a policy $\pi$ under $\mathcal{M}_{\mathrm{src}}$ and $\mathcal{M}_{\mathrm{tar}}$ as below,
\begin{equation}
\begin{aligned}
    &\quad\left|J_{\mathcal{M}_{\mathrm{tar}}}(\pi) - J_{\mathcal{M}_{\mathrm{src}}}(\pi)\right|\\ &\leq C_1\cdot\underbrace{\sup_{s,a}\left[D_{\mathrm{TV}}(P_{\mathrm{src}}(\cdot|s,a),P_{\mathrm{tar}}(\cdot|s,a))\right]}_{\mathrm{dynamics \, misalignment}},
\end{aligned}
\end{equation}
where $C_1=\frac{2\gamma r_{\max}}{(1-\gamma)^2}$ is a positive constant.
\end{lemma}

According to Lemma~\ref{lemma:1}, the performance difference is bounded by the dynamics misalignment between the source and target domains. Thus, selectively sharing source domain samples with smaller dynamics misalignment can tighten the performance difference, which builds the theoretical foundation for prior studies~\citep{wen2024contrastive, xu2024cross, lyu2024cross, lyu2025cross}. However, a critical limitation of this analysis is that \textbf{it misaligns with the RL objective}, that is, obtaining a policy $\pi$ to maximize $J_{\mathcal{M}_\mathrm{tar}}(\pi)$. Therefore, tightening such a performance difference bound does not necessarily lead to a well-performing policy in the target domain. Instead, it is more reasonable to narrow the sub-optimality gap of a policy $\pi$ trained on the source domain and evaluated on the target domain. Specifically, we denote the optimal policy in $\mathcal{M}_\mathrm{src}$ as $\pi^\star_\mathrm{src}$, and the in-sample optimal policy~\cite{kostrikov2021offline} extracted from the source domain dataset as $\pi^\star_\text{insrc}$. We define $\epsilon_\mathrm{src}^\star:=|J_{\mathcal{M}_\mathrm{src}}(\pi^\star_\mathrm{src})-J_{\mathcal{M}_\mathrm{src}}(\pi^\star_\mathrm{insrc})|$ to as the statistical error of the inherent performance difference between $\pi^\star_\text{src}$ and $\pi^\star_\text{insrc}$ on the source domain, which is constant given a source domain dataset. We aim to minimize the sub-optimality gap of $\pi$ in the target domain: $\mathrm{SubOpt}(\pi):= |J_{\mathcal{M}_{\mathrm{tar}}}(\pi)-J_{\mathcal{M}_{\mathrm{tar}}}(\pi^\star_{\mathrm{tar}})|$, where $\pi^\star_\mathrm{tar}$ is the optimal policy in the target domain. We derive the upper bound of this sub-optimality gap in Proposition~\ref{proposition:1}.

\begin{proposition}[\textbf{Sub-optimality gap on target domain}]
\label{proposition:1}
Denote the MDP of the source domain and target domain as $\mathcal{M}_{\mathrm{src}}$ and $\mathcal{M}_{\mathrm{tar}}$. For a policy $\hat{\pi}$ trained on $\mathcal{M}_{\mathrm{src}}$, the sub-optimality gap of $\hat{\pi}$ on $\mathcal{M}_{\mathrm{tar}}$ can be bounded as below,
\begin{equation}
\begin{aligned}
        &\mathrm{SubOpt}(\hat{\pi})\leq \underbrace{\left|J_{\mathcal{M}_{\mathrm{src}}}(\hat\pi)-J_{\mathcal{M}_{\mathrm{src}}}(\pi^\star_{\mathrm{insrc}})\right|}_{\mathrm{(a) \, value \, misalignment}} + \\ &~~~~~~~C_2\cdot\underbrace{\sup_{s,a}\left[D_{\mathrm{TV}}(P_{\mathrm{src}}(\cdot|s,a),P_{\mathrm{tar}}(\cdot|s,a))\right]}_{\mathrm{(b) \, dynamics \,misalignment}}+\epsilon_\mathrm{src}^\star,
\end{aligned}
\end{equation}
where $C_2=\frac{(2\gamma+2)r_{\max}}{(1-\gamma)^2}$ is a positive constant.

\end{proposition}
According to Proposition~\ref{proposition:1}, the sub-optimality gap in the target domain can be controlled by three terms: (a) \textit{value misalignment}, representing the sub-optimality on the source domain; (b) \textit{dynamics misalignment} as considered in previous works; (c) a constant statistical error $\epsilon_\mathrm{src}^\star$. To tighten such a sub-optimality gap, both dynamics and value misalignment need to be minimized. Hence, we can answer the previous question: \textit{both dynamics and value alignment are essential for efficient cross-domain offline RL}.

\section{Dynamics- and Value-Aligned Data Filtering}

Proposition~\ref{proposition:1} conveys a promising way to learn a well-performing policy in the target domain: \textbf{(1)} minimize the dynamics misalignment between the source domain and target domain; \textbf{(2)} minimize the value misalignment between the learned policy and the in-sample optimal policy in the source domain. Neglecting either factor would compromise policy performance. To address this, we adopt a data filtering paradigm inspired by prior works~\citep{xu2024cross, wen2024contrastive,lyu2025cross}, which retains source domain data that exhibits aligned dynamics with the target domain, and on which the learned policy can be close to the in-sample optimal policy in the source domain.  

\subsection{Advantage-Aware Value Alignment}
\label{sec:advantage_value}
Given that several existing methods can be applied for measuring dynamics misalignment, such as contrastive learning~\citep{wen2024contrastive} and optimal transport~\citep{lyu2025cross}, the crucial part remains how to capture value misalignment. To tackle this problem, we derive an upper bound of the value misalignment term, which gives us insight for the practical solution.

\begin{proposition}[\textbf{Value Misalignment}]
\label{proposition:2}

Denote the MDP of the source domain and target domain as $\mathcal{M}_{\mathrm{src}}$ and $\mathcal{M}_{\mathrm{tar}}$, and the behavior policy of the $\mathcal{D}_\mathrm{src}$ as $\mu$. For policy $\hat\pi$ trained on $\mathcal{D}_\mathrm{src}$ using an in-sample offline RL algorithm (e.g., IQL), we assume that $\left(\hat\pi(a|s)-\mu(a|s)\right)A_\mu(s,a)\geq0$. Then the value misalignment in Proposition~\ref{proposition:1} can be upper bounded as follows:
    \begin{align}
    &\left|J_{\mathcal{M}_{\mathrm{src}}}(\hat\pi)-J_{\mathcal{M}_{\mathrm{src}}}(\pi^\star_{\mathrm{insrc}})\right| \leq  \label{eq:1}\\ 
    &  ~~-\mathbb{E}_{s\sim\rho_\mu(\cdot),a\sim\mu(\cdot|s)}\left[A_{\pi^\star_{\mathrm{insrc}}}(s,a)\right]+ \mathcal{O}\left(\frac{1}{(1-\gamma)^2}\right), \nonumber
    \end{align}
where $A(s,a)$ is the advantage function, and $\rho_\mu(\cdot)$ is the state visiting distribution under $\mu$.

\end{proposition}

\textbf{Remark.} The assumption $\left(\hat\pi(a|s)-\mu(a|s)\right)A_\mu(s,a)\geq0$ indicates that, for any state-action pair $(s,a)$, if $A_\mu(s,a)\geq0$, i.e., the action $a$ shows superiority over average, the policy $\hat\pi$ has a higher probability than $\mu$ to choose action $a$, and vice versa. Theoretically, this assumption guarantees that the learned policy $\hat\pi$ enjoys a better performance than the behavior policy $\mu$ on the source domain (Proposition 4.1 in~\citet{liu2024adaptive}). This is reasonable since the goal of offline RL is to outperform the behavior policy. Furthermore, this assumption can be easily met. If we use IQL~\citep{kostrikov2021offline} to optimize the policy $\hat\pi$, since IQL utilizes exponential advantage weighted imitation learning, it explicitly updates the policy to favor actions with higher advantage values:
\begin{align*}
\pi_{k+1}=\arg\max_{\pi\in\Pi}\mathbb{E}_{(s,a)\in\mathcal{D}}\left[\exp(\alpha\cdot A_{\pi_k}(s,a))\log\pi(a|s)\right],
\end{align*}
then if the learned policy $\pi$ is initialized as $\mu$, after one step of policy update, any $(s,a)$ with $A_\mu(s,a)>0$ will be given more weight for imitation, and vice versa. This naturally satisfies our assumption. Therefore, our assumption is reasonable and easy to meet.

The right-hand side in Equation~\ref{eq:1} consists of \textbf{(1)} the negative in-sample optimal advantage value on the source domain under the behavior policy's state-action distribution, and \textbf{(2)} a bounded term. Proposition~\ref{proposition:2} gives an important insight that the value misalignment can be upper bounded by the negative advantage value under source domain offline data, estimated by the in-sample optimal advantage function on the source domain. Given that $J_{\mathcal{M}_{\text{src}}}(\hat\pi)-J_{\mathcal{M}_{\text{src}}}(\pi^\star_{\text{insrc}})\leq0$ holds since $\hat\pi$ is trained on $\mathcal{D}_\mathrm{src}$ using an in-sample offline RL algorithm, if we want to minimize the value misalignment, we need to maximize $\mathbb{E}_{s\sim\rho_\mu(\cdot),a\sim\mu(\cdot|s)}\left[A_{\pi^\star_\text{insrc}}(s,a)\right]$. This motivates our use of the in-sample optimal advantage function as a quantitative measure for value misalignment, where higher advantage values correspond to lower degrees of value misalignment.

\textbf{Remark.} It is worth noting that VGDF~\citep{xu2024cross} also emphasizes the importance of value alignment for cross-domain RL and leverages the value function to guide data filtering. However, DVDF fundamentally differs from VGDF in how it interprets value alignment. VGDF defines value alignment as minimizing $\left|V(s^\prime_\mathrm{src})-V(s^\prime_\mathrm{tar})\right|$, which quantifies dynamics discrepancy from a value difference perspective. That is, \textbf{VGDF still only addresses dynamics mismatch}. In contrast, DVDF minimizes $\left|J_{\mathcal{M}_\mathrm{src}}(\pi)-J_{\mathcal{M}_\mathrm{src}}(\pi^\star_\mathrm{insrc})\right|$ for value alignment, which explicitly measures policy optimality in the source domain, orthogonal to the dynamics mismatch. Thus, DVDF comprehensively considers both value and dynamics alignment, thereby distinguishing it from VGDF.

\subsection{Practical Implementation}
\label{sec:practical}

In Section~\ref{sec:advantage_value}, we have demonstrated that an in-sample optimal advantage function could be leveraged for capturing value misalignment. The next question is how to obtain such an advantage function. Since we cannot directly acquire the in-sample optimal advantage function, we propose leveraging a pre-trained offline policy trained on the source domain dataset to approximate the in-sample optimal policy, and using its corresponding advantage function to approximate the in-sample optimal advantage function. However, the advantage approximation error is non-negligible and must be minimized. This raises the question of how to perform offline pre-training effectively. We denote the pre-trained policy as $\pi_\text{pre}$, and its advantage function as $A_\text{pre}$. We also obtain an advantage function during pre-training, which we denote by $\hat{A}_\text{pre}$. Note that $A_\text{pre}$ and $\hat{A}_\text{pre}$ are typically different due to the additional conservatism introduced in offline RL. We analyze the advantage approximation error in Proposition~\ref{proposition:3}, which provides guidance on how to properly conduct offline pre-training.


\begin{proposition}
[\textbf{Advantage Approximation Error}]
\label{proposition:3}

Given a pre-trained policy $\pi_{\mathrm{pre}}$ and advantage function $\hat{A}_{\mathrm{pre}}(\cdot)$ to approximate the in-sample optimal policy $\pi^\star_\mathrm{insrc}$ and in-sample optimal advantage function $A_{\pi^\star_\mathrm{insrc}}(\cdot)$, then the advantage approximation error on offline samples generated by the behavior policy $\mu$ is
\begin{equation}
\begin{aligned}
    &\quad \mathbb{E}_{s\sim\rho_\mu(\cdot),a\sim\mu(\cdot|s)}\left[\hat{A}_{\mathrm{pre}}(s,a)-A_{\pi^\star_\mathrm{insrc}}(s,a)\right]= \\ &~~\Delta J_{\mathcal{M}_\mathrm{src}}(\pi^\star_\mathrm{insrc},\pi_\mathrm{pre})+\mathbb{E}_{s\sim\rho_\mu(\cdot),a\sim\mu(\cdot|s)}\left[\Delta(s,a)\right],
\end{aligned}
\end{equation}
with $\Delta J_{\mathcal{M}_\mathrm{src}}(\pi^\star_\mathrm{insrc},\pi_\mathrm{pre})=J_{\mathcal{M}_\mathrm{src}}(\pi_\mathrm{insrc}^\star)-J_{\mathcal{M}_\mathrm{src}}(\pi_\mathrm{pre})$, and $\Delta(s,a)=\hat{A}_{\mathrm{pre}}(s,a)-{A}_{\mathrm{pre}}(s,a)$.
\end{proposition}

Proposition~\ref{proposition:3} suggests that minimizing the advantage approximation error requires selecting an offline RL algorithm with two properties: (1) strong empirical performance (to minimize $\Delta J_{\mathcal{M}_\mathrm{src}}(\pi^\star_\mathrm{insrc},\pi_\mathrm{pre})$), and (2) accurate advantage estimation (such that $\Delta(s,a)$ is minimized). Although IQL is a natural candidate due to its ability to achieve near in-sample optimal performance across diverse benchmarks and its straightforward advantage estimation, its known tendency for $V$-function underestimation (caused by suboptimal actions)~\citep{xu2023offline, chenactive} may compromise advantage accuracy and consequently mislead data filtering. To address this limitation while maintaining high performance, we instead adopt another offline RL algorithm, Sparse Q-Learning (SQL)~\citep{xu2023offline}, for pre-training. As an in-sample learning algorithm that explicitly enforces policy sparsity, SQL achieves both near in-sample optimal performance and more reliable advantage estimates, thereby better satisfying our dual requirements of algorithmic performance and advantage accuracy.

{After pre-training, we obtain a $Q$-function $\hat{Q}_\mathrm{pre}(s,a)$ and a $V$-function $\hat{V}_\mathrm{pre}(s)$, we directly derive the advantage function as $\hat{A}_\mathrm{pre}(s,a)=\hat{Q}_\mathrm{pre}(s,a)-\hat{V}_\mathrm{pre}(s)$.} Then, we can leverage the pre-trained advantage function as an indicator of value misalignment. The next step is to choose the indicator of the dynamics misalignment. We can just follow previous studies, and apply methods such as contrastive learning~\citep{wen2024contrastive} and optimal transport~\citep{lyu2025cross}. Here, we follow IGDF~\citep{wen2024contrastive} and measure dynamics misalignment via contrastive learning. Specifically, we train a score function $h(s,a,s^\prime)$ via the NCE loss:
\begin{align*}
    \mathcal{L}_{\text{NCE}}=-\mathbb{E}_{(s,a,s^\prime_\text{tar})}\mathbb{E}_{S^\prime_\text{src}}\left[\log\frac{h(s,a,s^\prime_\text{tar})}{\sum_{s^\prime\in {\{s^\prime_\text{tar}\}\cup S^\prime_\text{src}}}h(s,a,s^\prime)}\right],
\end{align*}
{where $S^\prime_\mathrm{src}$ represents the next states from the source dataset.} Intuitively, $h(s,a,s^\prime)$ assigns high scores when $s^\prime\sim P_{\mathcal{M}_\text{tar}}(\cdot|s,a)$, and assigns low scores when $(s,a)\in\mathcal{D}_\text{tar}$ and $s^\prime\in\mathcal{D}_\text{src}$. Hence, $h(s,a,s^\prime)$ can reflect whether the dynamics of the transition $(s,a,s^\prime)$ aligns with the target domain dynamics.

\begin{table*}[t]
    \centering
    \begingroup
    \setlength{\tabcolsep}{11pt}
    \caption{\textbf{Performance comparison under kinematic shifts.} half=halfcheetah, hopp=hopper, walk=walker2d, r=random, m=medium, me=medium-expert, mr=medium-replay, e=expert. We report the normalized score evaluated in the target domain after 1M steps of training, and $\pm$ captures the standard deviation across 5 seeds. We \textbf{bold} the highest scores for each task.}
    \label{tab:mainresults}
    \begin{tabular}{l|ccc|cc|cc}
    \toprule
    \textbf{Dataset} & IQL & BOSA & DARA & IGDF & DVDF-IGDF & OTDF & DVDF-OTDF\\
    \midrule
    half-r & 4.9 & 2.2 & 4.7 & \textbf{5.4$\pm$0.4} & 4.6$\pm$0.1 & \textbf{2.2$\pm$0.2} & 1.7$\pm$0.1 \\
    half-m & \textbf{45.2} & 39.6 & 44.1 & \textbf{45.2$\pm$0.1} & 45.1$\pm$0.2 & 42.2$\pm$0.1 & \textbf{45.4$\pm$0.6} \\
    half-mr & 22.1 & \textbf{26.3} & 21.6 & 22.9$\pm$1.4 & \textbf{26.6$\pm$2.3} & 15.6$\pm$3.1 & \textbf{26.8$\pm$4.4} \\
    half-me & 43.7 & 42.2 & 52.7 & 57.1$\pm$8.9 & \textbf{66.7$\pm$6.3} & \textbf{46.7$\pm$4.4} & 45.9$\pm$3.0 \\
    half-e & 49.7 & \textbf{84.3} & 47.4 & 47.6$\pm$2.1 &\textbf{58.8$\pm$4.7} & 79.6$\pm$3.0 & \textbf{88.9$\pm$5.6} \\
    hopp-r & 4.5 & \textbf{40.7} & 3.8 & \textbf{13.0$\pm$1.9} & 3.3$\pm$0.1 & 2.9$\pm$0.4 & \textbf{12.6$\pm$0.1} \\
    hopp-m & 48.8 & \textbf{71.4} & 48.8 & 54.3$\pm$6.6 & \textbf{59.1$\pm$3.4} & 46.3$\pm$3.7 & \textbf{67.8$\pm$4.1} \\
    hopp-mr & 40.2 & 29.5 & 41.6 & 30.0$\pm$5.2 & \textbf{32.1$\pm$0.8} & 26.2$\pm$4.4 & \textbf{44.7$\pm$2.2} \\
    hopp-me & 12.5 & 49.6 & 17.0 & 11.6$\pm$0.6 & \textbf{60.2$\pm$5.9} & 58.1$\pm$4.9 & \textbf{70.2$\pm$7.7} \\
    hopp-e & 62.6 & 94.8 & 59.1 & 70.1$\pm$3.2 & \textbf{83.9$\pm$5.0} & 97.0$\pm$3.3 & \textbf{111.8$\pm$4.5} \\
    walk-r & 4.0 & 2.2 & 5.1 & 5.2$\pm$0.3 & \textbf{9.8$\pm$1.7} & 0.0$\pm$0.0 & 0.0$\pm$0.0 \\
    walk-m & 48.7 &44.5 & 43.4 & 51.8$\pm$2.4 & \textbf{69.7$\pm$4.4} & 43.0$\pm$2.1 & \textbf{71.6$\pm$5.9} \\
    walk-mr & 12.6 & 4.8 & 15.6 & 11.2$\pm$1.1 & \textbf{22.6$\pm$1.8} & 10.7$\pm$1.9 & \textbf{25.6$\pm$2.4} \\
    walk-me & 95.4 & 35.1 & 85.3 & 90.6$\pm$3.4 & \textbf{104.6$\pm$5.1} & 63.1$\pm$6.6 & \textbf{91.6$\pm$8.2} \\
    walk-e & 90.1 & 41.9 & 85.5 & 93.7$\pm$5.8 & \textbf{108.0$\pm$4.3} & 98.9$\pm$2.1 & \textbf{106.0$\pm$1.2} \\
    ant-r & 11.5 & \textbf{31.5} & 10.9 & 13.7$\pm$1.9 & \textbf{15.6$\pm$2.2} & 11.6$\pm$1.0 & \textbf{25.7$\pm$3.4} \\
    ant-m & 89.9 & 28.4 & \textbf{98.9} & 88.0$\pm$4.6 & \textbf{98.1$\pm$5.0} & 86.1$\pm$3.7 & \textbf{97.1$\pm$5.0} \\
    ant-mr & 46.8 & 22.0 & 42.1 & \textbf{58.2$\pm$9.7} & 44.1$\pm$7.6 & \textbf{39.6$\pm$8.1} & 36.8$\pm$3.9 \\
    ant-me & 106.1 & 102.5 & 104.8 & 112.8$\pm$4.0 & \textbf{126.6$\pm$7.4} & 105.1$\pm$3.9 & \textbf{117.2$\pm$6.1} \\
    ant-e & 111.0 & 57.6 & 115.1 & 119.2$\pm$5.6 & \textbf{125.2$\pm$3.9} & \textbf{111.6$\pm$2.9} & 107.9$\pm$4.0 \\ 
    \midrule
    \textbf{Total} & 950.3 & 851.1 & 947.5 & 1001.6 & \textbf{1164.7} & 986.5 & \textbf{1172.3} \\ 
    \bottomrule
    \end{tabular}
    \endgroup

\end{table*}

Based on the pre-trained advantage function $\hat{A}_{\text{pre}}(\cdot)$ and score function $h(\cdot)$, we propose a practical algorithm, termed DVDF (\textbf{D}ynamics- and \textbf{V}alue-aligned \textbf{D}ata \textbf{F}iltering), which selectively shares source domain data with smaller dynamics and value misalignment to train a target policy. Specifically, we define a new score function $g(s,a,s^\prime)$ as follows:
\begin{equation}
\label{eq:g}
    g(s,a,s^\prime) = \lambda\cdot h(s,a,s^\prime) + (1-\lambda)\cdot \text{Norm}({\hat{A}_{\text{pre}}(s,a)}),
\end{equation}
where $\lambda$ is a tunable hyperparameter, and $\text{Norm}(\cdot)$ is the min-max normalization operator. $g(\cdot)$ balances value and dynamics misalignment through a simple weighted summation strategy. \textbf{This design directly aligns with our theoretical results in Proposition~\ref{proposition:1},} which also combines the two terms via a weighted summation. Then, we extract the top $\xi$-quantile of batch source samples for training, and weigh the Temporal-Difference (TD)-error of selected source data using the score function as in~\citep{wen2024contrastive}:
\begin{equation}
\begin{aligned}
        &\mathcal{L}_{Q}(\theta)=\frac{1}{2}\mathbb{E}_{(s,a,s^\prime)\sim\mathcal{D}_\text{tar}}\left[(Q_\theta-\mathcal{T}Q_\theta)^2\right]+
        \\ &~~\frac{1}{2} \mathbb{E}_{(s,a,s^\prime)\sim\mathcal{D}_\text{src}}\left[w(s,a,s^\prime)g(s,a,s^\prime)\right(Q_\theta-\mathcal{T}Q_\theta)^2],
\end{aligned}
\end{equation}
where $w(s,a,s^\prime) = \mathbb{I}(g(s,a,s^\prime)>g_{\xi\%})$ is an indicator function, {and $g_{\xi\%}$ means the $\xi$-th quantile of the $g$-values among source domain samples in a batch.} The last step is to update the policy via offline RL algorithms such as IQL. Note that DVDF can serve as a plug-in module and can be combined with different cross-domain offline RL algorithms such as IGDF and OTDF~\citep{lyu2025cross}, yielding DVDF-IGDF and DVDF-OTDF. We present the detailed algorithm procedure of DVDF-IGDF and DVDF-OTDF in Appendix~\ref{appendix:pseudocode}.



\section{Experiments}

In this section, we examine the effectiveness of our proposed method by conducting extensive experiments on environments with various challenging dynamics shifts. We compare the performance of DVDF and other baselines in Section~\ref{sec:main}, and empirically show that DVDF achieves effective offline policy adaptation and consistently outperforms prior strong baselines across varied dynamics shifts and dataset qualities. In Section~\ref{sec:ablation} and Section~\ref{sec:parameter}, we conduct a detailed ablation study and parameter study for a better understanding of DVDF.

\subsection{Main Results under Various Dynamics Shifts}
\label{sec:main}

\begin{figure*}
    \centering
	\subfigure[Performance comparison with SQL and IQL pre-trained advantage function.] {\includegraphics[width=.98\textwidth]{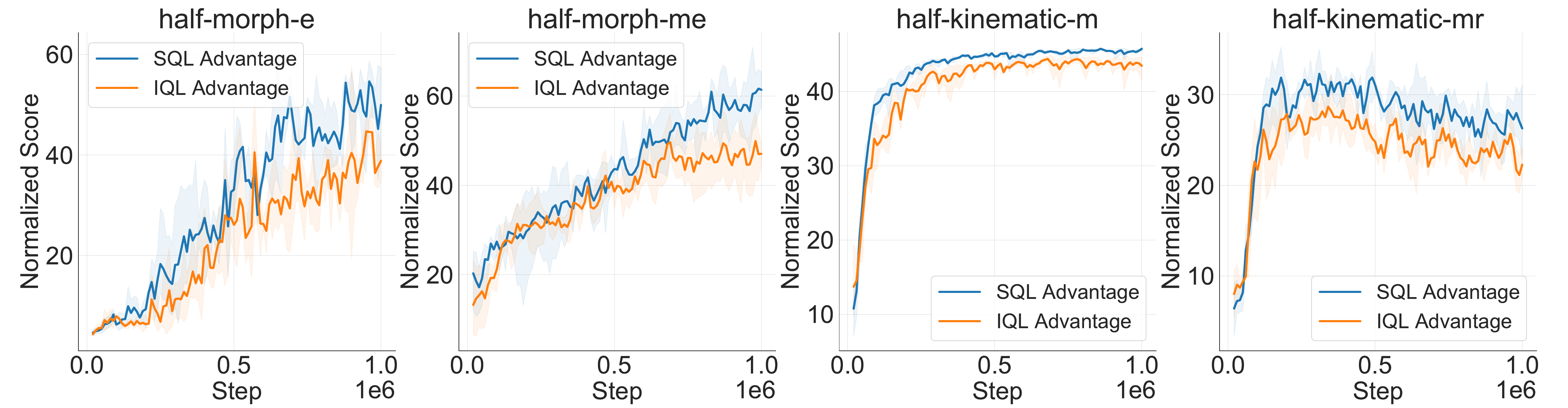}}
	\subfigure[Advantage estimation error comparison with SQL and IQL pre-trained advantage function.] {\includegraphics[width=.98\textwidth]{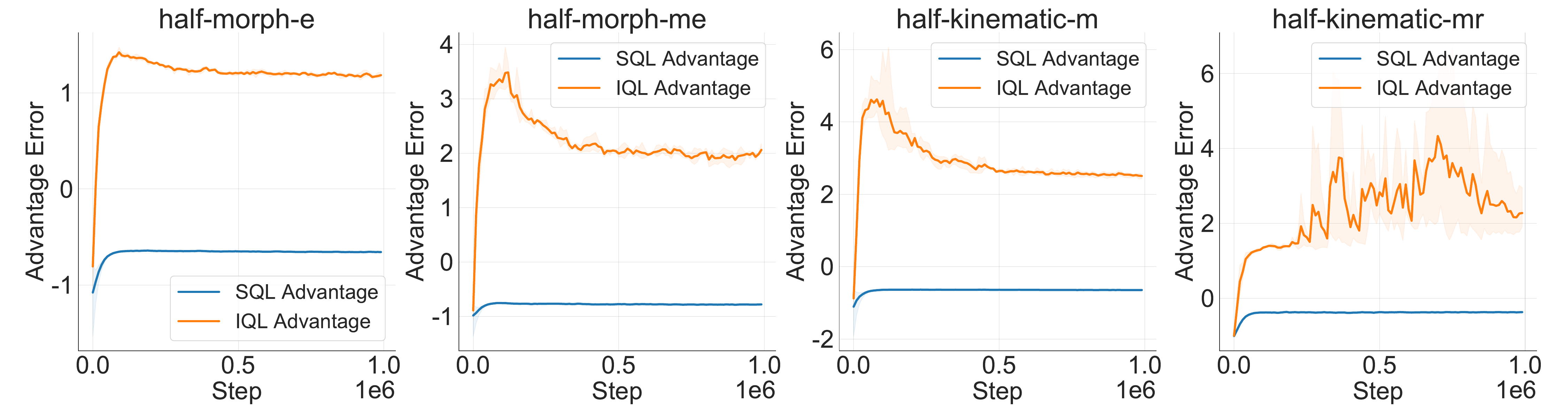}
     }
    \caption{Ablation study on SQL pre-trained advantage function.}
    \label{fig:ablation_main}
\end{figure*}

\textbf{Tasks and Datasets.} For the types of dynamics shifts, we consider kinematic shifts and morphology shifts in this paper, and the dynamics shifts are applied to four tasks (\texttt{halfcheetah}, \texttt{hopper}, \texttt{walker2d}, \texttt{ant}) from OpenAI Gym~\citep{brockman2016openai}. The kinematic shifts are realized by reducing the rotation range of some joints, and the morphology shifts are simulated by modifying the size of some limbs. We defer more details for the realization of dynamics shifts to Appendix~\ref{appendix:realizationkinematic} and~\ref{appendix:realizationmorph}. Since only a limited amount of target data is accessible, we can sample a percentage of data from offline datasets from D4RL~\citep{fu2020d4rl} as the target domain datasets. We set the percentage to $10\%$ in our experiments. For source domain datasets, we collect data in the modified environments, following a similar data collection process as D4RL. Specifically, we collect datasets of five data qualities (\texttt{random}, \texttt{medium}, \texttt{medium-replay}, \texttt{medium-expert}, \texttt{expert}) with an SAC~\citep{haarnoja2018soft} agent trained to different levels of performance in the respective environments, and each source domain dataset contains around 1M samples, much more than the target domain datasets. This amounts to a total of \textbf{40} source domain datasets and \textbf{20} target domain datasets. Note that for each pair of source and target domain datasets, the type of tasks and dataset quality remain the same, and the difference lies in the transition dynamics and the dataset size. We also examine DVDF in extremely low-data settings~\citep{lyu2025cross} where the target dataset contains only 5,000 transitions. We defer the results in Appendix~\ref{appendix:extended}.

\textbf{Baselines.} We choose the following baselines for comparison: \textbf{IQL}~\citep{kostrikov2021offline} (which we train directly on the mixture of source domain data and target domain data). \textbf{BOSA}~\citep{liu2024beyond}, \textbf{DARA}~\citep{liu2022dara}, \textbf{IGDF}~\citep{wen2024contrastive} and \textbf{OTDF}~\citep{lyu2025cross}. The backbone of IGDF and OTDF is IQL. We exclude VGDF~\citep{xu2024cross} as our baseline since it requires an online target environment. More details about these baselines are presented in Appendix~\ref{appendix:baseline}. For our method, we implement DVDF-IGDF and DVDF-OTDF for comparison.

\begin{figure*}
    \centering
	\subfigure[Effect of $\lambda$] {\includegraphics[width=.98\textwidth]{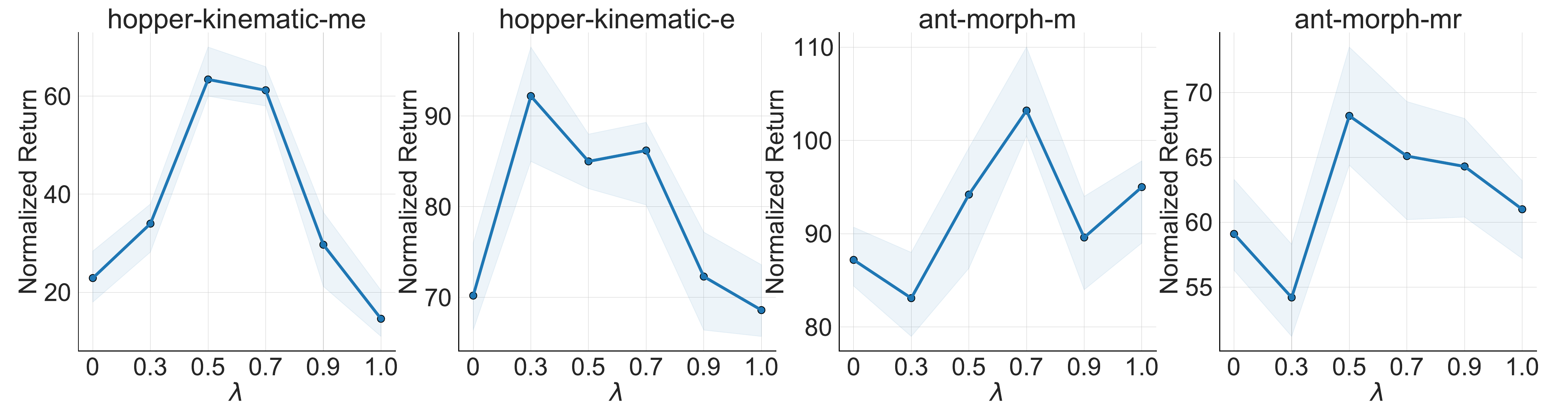}}
	\subfigure[Effect of $\xi$] {\includegraphics[width=.98\textwidth]{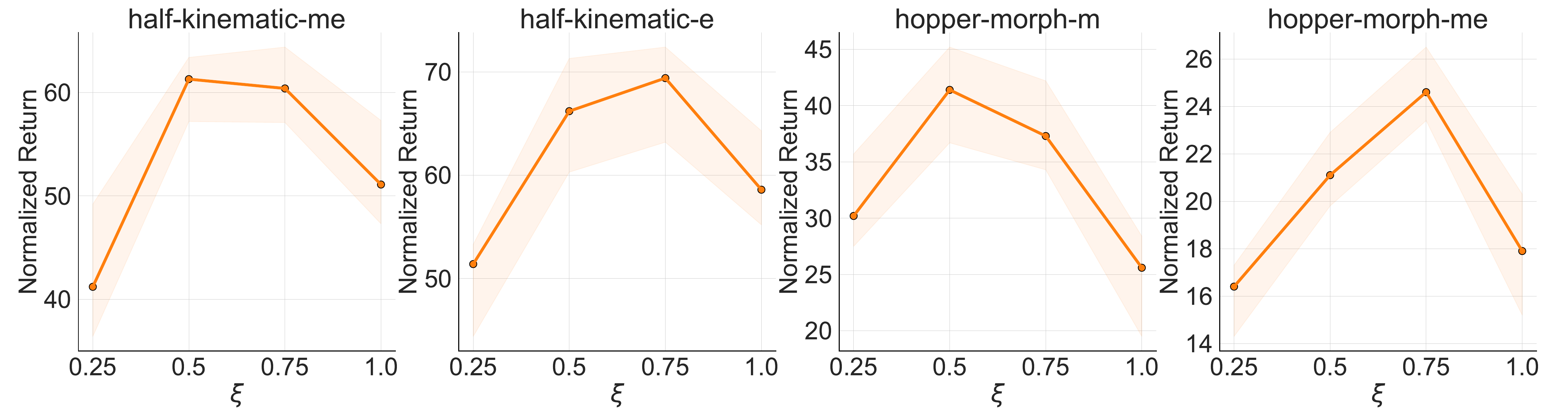}
     }
    \caption{Parameter sensitivity experiments on $\lambda$ and $\xi$.}
    
    \vspace{-0.35cm}
    \label{fig:parameter}
\end{figure*}

\textbf{Experimental Results.} We present the comparison results for each method under kinematic shifts in Table~\ref{tab:mainresults}. Due to space limit, the results under morphology shifts are deferred to Appendix~\ref{appendix:morphresults}. We report the normalized score in the target domain. Empirical results demonstrate that DVDF consistently enhances the performance of base algorithms (IGDF and OTDF) while outperforming more baselines (IQL, BOSA, and DARA) across diverse tasks and dataset qualities under kinematic shifts. Notably, DVDF-IGDF surpasses IGDF on \textbf{16} out of 20 tasks under kinematic shifts, and DVDF-OTDF achieves a higher score than OTDF on \textbf{15} out of 20 tasks under kinematic shifts. Adopting DVDF incurs an increase of $\textbf{16.3\%}$ (from 1001.6 to 1164.7) and $\textbf{18.8\%}$ (from 986.5 to 1172.3) to IGDF and OTDF respectively, in terms of the total normalized score under kinematic shifts. Moreover, DVDF shows superiority especially on datasets containing high-quality samples (\texttt{medium-expert} and \texttt{expert}), outperforming IGDF on \textbf{8} out of 8 tasks and OTDF on \textbf{6} out of 8 tasks. We attribute this to DVDF's ability to select dynamics- and value-aligned source domain data, whereas IGDF and OTDF may discard substantial value-aligned data, which inevitably limits their potential for effective policy transfer.

\subsection{Ablation Study}
\label{sec:ablation}

In this section, we examine the necessity of SQL pre-training for obtaining the advantage function. We pre-train both SQL and IQL, and apply the resulting advantage functions for data filtering. Experiments are conducted on four datasets using DVDF-IGDF as the base algorithm, with other settings consistent with Section~\ref{sec:main}. We compare the two pre-training methods in terms of algorithm performance and advantage estimation error.

\textbf{Performance Comparison.} {Figure~\ref{fig:ablation_main} (a) shows the learning curves and performance comparison on four datasets using SQL and IQL pre-trained advantage functions. Clearly, employing the SQL pre-trained advantage function for data filtering yields better performance than using IQL. Therefore, we use the advantage function obtained by SQL pre-training for data filtering in our experiments.}

\textbf{Advantage Estimation Error Comparison.} {We next examine whether the performance improvement from SQL pre-training stems from a more accurate advantage function. We define the advantage estimation error as $\mathcal{E}=\mathbb{E}_{(s,a)\in\mathcal{D}_\mathrm{src}}\frac{\hat{A}(s,a)-A(s,a)}{A(s,a)}$, where $\hat{A}(\cdot)$ is the estimated advantage function by either SQL or IQL, and $A(\cdot)$ is the true advantage function obtained by Monte Carlo rollouts. Figure~\ref{fig:ablation_main} (b) shows the advantage estimation error comparison during pre-training. We observe that IQL quickly overestimates the advantage value, consistent with~\cite{xu2023offline}, noting that IQL tends to underestimate the 
$V$-function, thereby inflating the estimated advantage. In contrast, SQL typically underestimates the advantage due to the sparsity induced in $V$-function learning. Nevertheless, SQL maintains a smaller estimation error than IQL throughout pre-training, indicating that SQL provides more accurate advantage estimation.}




\subsection{Parameter Sensitivity}
\label{sec:parameter}
In this section, we investigate the sensitivity of DVDF to the hyperparameters. There are two main hyperparameters in DVDF: the data selection ratio $\xi$ and the alignment tradeoff coefficient $\lambda$. We choose DVDF-IGDF as our base algorithm, running for 1M steps with 10 random seeds. The dataset setting follows Section~\ref{sec:main}.

\textbf{Alignment tradeoff coefficient $\lambda$.} {The parameter $\lambda$ balances the weight of dynamics alignment and value alignment during data filtering. A larger $\lambda$ emphasizes more on dynamics alignment, and vice versa. We vary $\lambda$ across $\{0.0, 0.3,0.5,0.7,0.9,1.0\}$ and conduct experiments on four different tasks. Figure~\ref{fig:parameter} (a) shows the impact of different values of $\lambda$ on the final performance, which indicates that neither excessive emphasis on dynamics alignment nor value alignment represents the best choice, and $\lambda=0.7$ could achieve an effective trade-off between dynamics and value alignment. Therefore, we fix $\lambda=0.7$ across all datasets in our experiments without further tuning.}

\textbf{Data selection ratio $\xi$.} {The parameter $\xi$ decides how many source domain samples can be shared. A larger $\xi$ implies more source domain samples are accepted. To examine its influence, we conduct experiments on four tasks. We sweep $\xi$ across $\{0.25, 0.5, 0.75, 1.0\}$ and present the final performance comparison in Figure~\ref{fig:parameter} (b). We observe an inferior performance when $\xi=0.25$ or $\xi=1.0$, and setting $\xi=0.5$ could achieve a favorable result on most tasks. Therefore, we set $\xi=0.5$ uniformly for DVDF in all the experiments, instead of performing task-specific tuning.}

\section{Conclusion}
In this paper, through empirical and theoretical analyses, we demonstrate that both dynamics alignment and value alignment are critical for efficient cross-domain offline RL. Building upon this insight, we propose a novel method, DVDF, which leverages a pretrained advantage function to quantify value misalignment and performs data filtering by jointly considering dynamics and value misalignment, filling a gap in prior research. Extensive experiments across various dynamics shift scenarios and tasks demonstrate that DVDF outperforms prior strong baselines and brings significant performance improvement to base algorithms.

\section*{Impact Statement}

This paper presents work whose goal is to advance the field of Machine
Learning. There are many potential societal consequences of our work, none
of which we feel must be specifically highlighted here.

\nocite{langley00}

\bibliography{example_paper}
\bibliographystyle{icml2026}

\newpage
\appendix
\onecolumn

\section{Related Works}

\textbf{Offline RL.} Typical offline RL~\citep{levine2020offline, prudencio2023survey} assumes only access to a static dataset collected in the target environment. Value overestimation may occur in offline RL due to the OOD action issue~\citep{kumar2020conservative, fujimoto2019off, fujimoto2021minimalist}. Common solutions for this issue include conservative value estimation~\citep{kumar2020conservative, lyu2022mildly, nikulin2023anti, cheng2022adversarially, jin2021is, zhang2024pessimism, Rashidinejad2022bridging}, adding policy constraints~\citep{kumar2019stabilizing, fujimoto2021minimalist, wu2021uncertainty}, and augmenting the dataset with dynamics models~\citep{yu2020mopo, yu2021combo, kidambi2020morel,qiao2026romi,qiao2025sumo}. Our focus is different from these works since we leverage data from another source domain for policy learning.

\textbf{Domain Adaptation in RL.} In this work, we investigate the cross-domain policy adaptation problem under dynamics shifts~\citep{xu2024cross, lyu2024cross, xue2023state,yan2026cross,qiao2026droco}, while keeping other MDP components unchanged. Previous studies for this problem include domain randomization~\citep{slaoui2019robust, mehta2020active}, system identification~\citep{clavera2018learning, du2021auto}, imitation learning~\citep{chae2022robust, kim2020domain}, and meta RL~\citep{finn2017model, nagabandi2018learning}. However, these methods require a manipulable simulator or expert trajectories from the target domain. Recent works~\citep{pan2024survive, guo2024off, guo2025mobody, wang2024return, niu2022trust, niu2023h2o+} dismiss these limitations and study the setting where limited target domain data and sufficient source domain data are available, either online or offline. Instead, we focus on a cross-domain offline RL setting, where both source and target domain data are offline. In this setting, recent studies include reward modification through domain classifier~\citep{liu2022dara} or decision transformer~\citep{wang2024return}, utilizing supported value optimization~\citep{liu2024beyond}, leveraging contrastive representation~\citep{wen2024contrastive} or optimal transport~\citep{lyu2025cross} for data filtering, and so on. {In addition, PSEC~\citep{liuskill} achieves effective policy adaptation by dynamically composing the parameters of pre-trained source and target domain policies, DmC~\citep{le2025dmc} employs a KNN-based estimator as a measure of dynamics gap, and utilizes the KNN proximity score as a guiding signal for diffusion-based data augmentation.} These works have primarily focused on dynamics alignment while neglecting the critical role of value alignment in the source domain. In contrast, DVDF jointly considers dynamics and value alignment, filling this gap in prior research.

\section{Proofs of Theoretical Results}
\label{appendix:proof}
In this section, we provide the detailed proofs of the theoretical results in the main text.

\subsection{Proof of Lemma~\ref{lemma:1}}





\begin{proof} The proof starts with
\begin{equation*}
    \begin{aligned}
    \left|J_{\mathcal{M}_{\text{tar}}}(\pi)-J_{\mathcal{M}_{\text{src}}}(\pi)\right| &= \left|\frac{\gamma}{1-\gamma}\mathbb{E}_{s,a\sim \rho^\pi_{\mathcal{M}_\text{tar}}}\left[\mathbb{E}_{s^\prime\sim P_\text{tar}}\left[V_{\mathcal{M}_\text{src}}^\pi(s^\prime)\right]-\mathbb{E}_{s^\prime\sim P_\text{src}}\left[V_{\mathcal{M}_\text{src}}^\pi(s^\prime)\right]\right]\right| \\
    &= \left|\frac{\gamma}{1-\gamma}\mathbb{E}_{s,a\sim \rho^\pi_{\mathcal{M}_\text{tar}}}\left[\int_{s^\prime}\left(P_\text{tar}(s^\prime|s,a)-P_\text{src}(s^\prime|s,a)\right)V_{\mathcal{M}_\text{src}}^\pi(s^\prime)\right]\right| \\
    &\leq \frac{\gamma}{1-\gamma}\mathbb{E}_{s,a\sim \rho^\pi_{\mathcal{M}_\text{tar}}}\left[\int_{s^\prime}\left|P_\text{tar}(s^\prime|s,a)-P_\text{src}(s^\prime|s,a)\right|V_{\mathcal{M}_\text{src}}^\pi(s^\prime)\right] \\
    &\leq \frac{\gamma\cdot r_{\max}}{(1-\gamma)^2}\mathbb{E}_{s,a\sim \rho^\pi_{\mathcal{M}_\text{tar}}}\left[\int_{s^\prime}\left|P_\text{tar}(s^\prime|s,a)-P_\text{src}(s^\prime|s,a)\right|\right],
\end{aligned}
\end{equation*}
where the first equality holds from the telescoping lemma~\citep{xu2018algorithmic}. Moreover, by the definition of total variance distance, we have
\begin{equation*}
    \begin{aligned}
    &\frac{\gamma\cdot r_{\max}}{(1-\gamma)^2}\mathbb{E}_{s,a\sim \rho^\pi_{\mathcal{M}_\text{tar}}}\left[\int_{s^\prime}\left|P_\text{tar}(s^\prime|s,a)-P_\text{src}(s^\prime|s,a)\right|\right]\\
    &\qquad=\frac{2\gamma r_{\max}}{(1-\gamma)^2}\mathbb{E}_{s,a\sim \rho^\pi_{\mathcal{M}_\text{tar}}}\left[D_{\mathrm{TV}}(P_\text{tar}(\cdot|s,a),P_\text{src}(\cdot|s,a))\right] \\
    &\qquad\leq \frac{2\gamma r_{\max}}{(1-\gamma)^2}\cdot \sup_{s,a}\left[D_{\mathrm{TV}}(P_\text{tar}(\cdot|s,a),P_\text{src}(\cdot|s,a))\right].
\end{aligned}
\end{equation*}
Combining the above two inequalities completes the proof.
\end{proof}

\textbf{Remark.} Let $C_1=\frac{2\gamma r_\mathrm{max}}{(1-\gamma)^2}$, which scales as $\mathcal{O}(\frac{1}{(1-\gamma)^2})$. This constant could be further reduced to $\mathcal{O}(\frac{1}{1-\gamma})$ to ensure a tighter performance bound, under the Lipschitz continuity assumption. Specifically, we introduce the following assumption and corollary.
\begin{assumption}[Lipschitz Continuity.] 
\label{assumption:lipschitz}
The learned $V$-function is $K_V$-Lipschitz, w.r.t. state $s$, i.e., $\forall s_1,s_2\in\mathcal{S}$, $|V(s_1)-V(s_2)|\leq K_V\left\|s_1-s_2\right\|$.
\end{assumption}

\begin{corollary}[Tighter Performance Bound.] Under Assumption~\ref{assumption:lipschitz}, the performance difference of a policy $\pi$ under $\mathcal{M}_\mathrm{src}$ and $\mathcal{M}_\mathrm{tar}$ admits a tighter bound as
\begin{align*}
    \left|J_{\mathcal{M}_\mathrm{src}}(\pi)-J_{\mathcal{M}_\mathrm{tar}}(\pi)\right|\leq C\cdot \sup_{s,a}\left[D_{TV}(P_\mathrm{src}(\cdot|s,a),P_\mathrm{tar}(\cdot|s,a))\right],
\end{align*}
where $C=\frac{\gamma}{1-\gamma}\cdot K_V$.
\end{corollary}
\begin{proof}
    The conclusion can be directly obtained by following the proof procedure of Theorem 4.5 of~\citet{ji2022update}.
\end{proof}

\subsection{Proof of Proposition~\ref{proposition:1}}



    

\begin{proof}
    We first decompose the desired performance bound into four parts:
    \begin{align*}
        \begin{aligned}
            &\left|J_{\mathcal{M}_\text{tar}}(\hat\pi)-J_{\mathcal{M}_\text{tar}}(\pi^\star_\text{tar})\right|\\ &= |\left(J_{\mathcal{M}_\text{src}}(\hat\pi)-J_{\mathcal{M}_\text{src}}(\pi^\star_\text{insrc})\right)+\left(J_{\mathcal{M}_\text{tar}}(\hat\pi)-J_{\mathcal{M}_\text{src}}(\hat\pi)\right)+\left(J_{\mathcal{M}_\text{src}}(\pi^\star_\text{src})-J_{\mathcal{M}_\text{tar}}(\pi^\star_\text{tar})\right)\\
            &+\left(J_{\mathcal{M}_\text{src}}(\pi^\star_\text{insrc})-J_{\mathcal{M}_\text{src}}(\pi^\star_\text{src})\right)| \\
            &\leq \underbrace{\left|J_{\mathcal{M}_\text{src}}(\hat\pi)-J_{\mathcal{M}_\text{src}}(\pi^\star_\text{insrc})\right|}_{{(\text{I})}} + \underbrace{\left|J_{\mathcal{M}_\text{tar}}(\hat\pi)-J_{\mathcal{M}_\text{src}}(\hat\pi)\right|}_{(\text{II})}+\underbrace{\left|J_{\mathcal{M}_\text{src}}(\pi^\star_\text{src})-J_{\mathcal{M}_\text{tar}}(\pi^\star_\text{tar})\right|}_{(\text{III})}\\
            &+\underbrace{\left|J_{\mathcal{M}_\text{src}}(\pi^\star_\text{insrc})-J_{\mathcal{M}_\text{src}}(\pi^\star_\text{src})\right|}_{(\text{IV})}.
        \end{aligned}
    \end{align*}

Part (\text{I}) is exactly the desired term (a), i.e., the sub-optimality on the source domain. To get term (b), we need to bound parts (\text{II}) and (\text{III}), respectively.

We first bound part (\text{II}). By directly using Lemma~\ref{lemma:1}, we have
\begin{align*}
    \begin{aligned}
        (\text{II})&\coloneqq\left|J_{\mathcal{M}_\text{tar}}(\hat\pi)-J_{\mathcal{M}_\text{src}}(\hat\pi)\right| \\
        &\leq \frac{2\gamma r_{\max}}{(1-\gamma)^2}\cdot \sup_{s,a}\left[D_{\mathrm{TV}}(P_\text{tar}(\cdot|s,a),P_\text{src}(\cdot|s,a))\right].
    \end{aligned}
\end{align*}

Next, we bound part (\text{III}), i.e., the performance discrepancy between the optimal policy of two different MDPs.

For part (\text{III}), according to the definition of $J_\mathcal{M}(\hat\pi)$, we have $J_\mathcal{M}(\hat\pi)=V^{\hat\pi}_{\mathcal{M},h=0}(s)\coloneqq\mathbb{E}_{s\sim\rho_\mathcal{M}}[V_\mathcal{M}^{\hat\pi}(s)]$. To get the performance bound between two optimal policies in two MDPs, we can turn to analyze the optimal value difference of two MDPs at horizon 0 as
\begin{equation}
        \label{eq:value_diff}
        (\text{III})\coloneqq\left|V^\star_{\mathcal{M}_\text{src},h=0}(s)-V^\star_{\mathcal{M}_\text{tar},h=0}(s)\right|.
\end{equation}

To analyze Equation~\ref{eq:value_diff}, we first consider the value difference at step $h-1$:

\begin{equation*}
    \begin{aligned}
        &V^\star_{\text{src},h-1}(s)-V^\star_{\text{tar},h-1}(s) \\
        &=\max_{a\in\mathcal{A}}\int_{s^\prime}P_\text{src}(s^\prime|s,a)\left(r(s,a)+\gamma V^\star_{\text{src},h}(s^\prime)\right)-\max_{a\in\mathcal{A}}\int_{s^\prime}P_\text{tar}(s^\prime|s,a)\left(r(s,a)+\gamma V^\star_{\text{tar},h}(s^\prime)\right)\\
        &=\int_{s^\prime}P_\text{src}(s^\prime|s,a_1)\left(r(s,a_1)+\gamma V^\star_{\text{src},h}(s^\prime)\right)-\int_{s^\prime}P_\text{tar}(s^\prime|s,a_2)\left(r(s,a_2)+\gamma V^\star_{\text{tar},h}(s^\prime)\right)\\
        &\leq \int_{s^\prime}P_\text{src}(s^\prime|s,a_1)\left(r(s,a_1)+\gamma V^\star_{\text{src},h}(s^\prime)\right)-\int_{s^\prime}P_\text{tar}(s^\prime|s,a_1)\left(r(s,a_1)+\gamma V^\star_{\text{tar},h}(s^\prime)\right)\\
        &=\int_{s^\prime}\left(P_\text{src}(s^\prime|s,a_1)-P_\text{tar}(s^\prime|s,a_1)\right) r(s,a_1) + \gamma \int_{s^\prime}\left(P_\text{src}(s^\prime|s,a_1) V^\star_{\text{src},h}(s^\prime)-P_\text{tar}(s^\prime|s,a_1) V^\star_{\text{tar},h}(s^\prime)\right)\\
        &\leq \max_{a\in\mathcal{A}}\int_{s^\prime}\left(P_\text{src}(s^\prime|s,a)-P_\text{tar}(s^\prime|s,a)\right)r(s,a) +\max_{a\in\mathcal{A}}\left[\gamma \int_{s^\prime}\left(P_\text{src}(s^\prime|s,a)V^\star_{\text{src},h}(s^\prime)-P_\text{tar}(s^\prime|s,a)V^\star_{\text{tar},h}(s^\prime)\right)\right],
    \end{aligned}
\end{equation*}
where in the second equality, we have $a_1=\arg\max_{a\in\mathcal{A}}\int_{s^\prime}P_\text{src}(s^\prime|s,a)\left(r(s,a)+\gamma V^\star_{\text{src},h}(s^\prime)\right)$, and $a_2=\arg\max_{a\in\mathcal{A}}\int_{s^\prime}P_\text{tar}(s^\prime|s,a)\left(r(s,a)+\gamma V^\star_{\text{tar},h}(s^\prime)\right)$. In addition, we have
\begin{align*}
        &\max_{a\in\mathcal{A}}\int_{s^\prime}\left(P_\text{src}(s^\prime|s,a)-P_\text{tar}(s^\prime|s,a)\right)r(s,a)\leq \max_{a\in\mathcal{A}}\int_{s^\prime}\left|P_\text{src}(s^\prime|s,a)-P_\text{tar}(s^\prime|s,a)\right|\cdot r_{\max}, \\ 
        & \max_{a\in\mathcal{A}}\left[\gamma \int_{s^\prime}\left(P_\text{src}(s^\prime|s,a)V^\star_{\text{src},h}(s^\prime)-P_\text{tar}(s^\prime|s,a)V^\star_{\text{tar},h}(s^\prime)\right)\right] \\
        &\quad \leq \gamma\max_{a\in\mathcal{A}}\int_{s^\prime} P_\text{tar}(s^\prime|s,a)\left(V^\star_{\text{src},h}(s^\prime)-V^\star_{\text{tar},h}(s^\prime)\right)+\gamma\max_{a\in\mathcal{A}}\int_{s^\prime}\left|P_\text{src}(s^\prime|s,a)-P_\text{tar}(s^\prime|s,a)\right|V^\star_{\text{src},h}(s^\prime).
\end{align*}
Therefore, combining them together, we obtain
\begin{align}
\begin{aligned}\label{eq:recursion}
  &V^\star_{\text{src},h-1}(s)-V^\star_{\text{tar},h-1}(s)\\
  &\quad \leq   \frac{2r_{\max}}{1-\gamma}\sup_{s,a}\left[D_{\mathrm{TV}}(P_\text{tar}(\cdot|s,a),P_\text{src}(\cdot|s,a))\right] + \gamma\max_{s^\prime\in\mathcal{S}}\left[V^\star_{\text{src},h}(s^\prime)-V^\star_{\text{tar},h}(s^\prime)\right].
\end{aligned}
\end{align}
If we denote
\begin{align*}
    a_h \coloneqq \max_{s\in\mathcal{S}}\left[V^\star_{\text{src},h}(s)-V^\star_{\text{tar},h}(s)\right]
\end{align*}
and 
\begin{align*}
    c\coloneqq\frac{2r_{\max}}{1-\gamma}\sup_{s,a}\left[D_{\mathrm{TV}}(P_\text{tar}(\cdot|s,a),P_\text{src}(\cdot|s,a))\right],
\end{align*}
then Equation~\ref{eq:recursion} can be simplified as
\begin{equation}
\label{eq:proof3}
    \begin{aligned}
    a_{h-1}&\leq c+\gamma a_h\\
    \Rightarrow a_{h-1}-\frac{c}{1-\gamma}&\leq \gamma\cdot\left(a_h-\frac{c}{1-\gamma}\right).
    \end{aligned}
\end{equation}

Note that Equation~\ref{eq:proof3} is a recursive expression. Repeating the process recursively, we can easily get
\begin{equation}
\label{eq:proof4}
    a_0-\frac{c}{1-\gamma}\leq\gamma^H\left(a_H-\frac{c}{1-\gamma}\right),
\end{equation}
where $H$ denotes the maximum step of an episode\footnote{We focus on the infinite-horizon setting, and $H$ only serves as an intermediate variable for analysis.}. According to the definition of the state value, the value of the terminal state is $0$, thus $a_H=0$. Plugging $a_H=0$ into Equation~\ref{eq:proof4}, we can get
\begin{align*}
    V^\star_{\text{src},0}(s)-V^\star_{\text{tar},0}(s)\leq a_0\leq c\cdot\frac{1-\gamma^H}{1-\gamma}.
\end{align*}

If we set $H\rightarrow\infty$, we have
\begin{align*}
    V^\star_{\text{src},0}(s)-V^\star_{\text{tar},0}(s)\leq \frac{2r_{\max}}{(1-\gamma)^2}\sup_{s,a}\left[D_{\mathrm{TV}}(P_\text{tar}(\cdot|s,a),P_\text{src}(\cdot|s,a))\right].
\end{align*}

Due to the interchangeability of $\mathcal{M}_\text{src}$ and $\mathcal{M}_\text{tar}$, we also have
\begin{align*}
        V^\star_{\text{src},0}(s)-V^\star_{\text{tar},0}(s)\geq -\frac{2r_{\max}}{(1-\gamma)^2}\sup_{s,a}\left[D_{\mathrm{TV}}(P_\text{tar}(\cdot|s,a),P_\text{src}(\cdot|s,a))\right].
\end{align*}

Therefore, 
\begin{align*}
    \begin{aligned}
        (\text{III})&\coloneqq\left|V^\star_{\mathcal{M}_\text{src},h=0}(s)-V^\star_{\mathcal{M}_\text{tar},h=0}(s)\right|\\
        &\leq \frac{2r_{\max}}{(1-\gamma)^2}\sup_{s,a}\left[D_{\mathrm{TV}}(P_\text{tar}(\cdot|s,a),P_\text{src}(\cdot|s,a))\right].
    \end{aligned}
\end{align*}

Combining the bounds of terms (\text{II}) and (\text{III}), we get
\begin{align*}
    (\text{II})+(\text{III})\leq C_2\cdot\sup_{s,a}\left[D_{\mathrm{TV}}(P_\text{tar}(\cdot|s,a),P_\text{src}(\cdot|s,a))\right],
\end{align*}
where $C_2=\frac{(2\gamma+2)r_{\max}}{(1-\gamma)^2}$.

For the term (\text{IV}), its value is exactly $\epsilon_\text{src}^\star$ by the definition of $\epsilon_\text{src}^\star$. This concludes the proof.
\end{proof}

\textbf{Remark.} Similar to Lemma~\ref{lemma:1}, $C_2$ could be further reduced to $C=\frac{2\gamma}{1-\gamma}\cdot K_V$ under the Lipschitz continuity assumption.

\subsection{Proof of Proposition~\ref{proposition:2}}




\begin{proof}

We first divide the sub-optimality on the source domain into two terms:
\begin{align*}
    J_{\mathcal{M}_\text{src}}(\hat\pi)-J_{\mathcal{M}_\text{src}}(\pi^\star_\text{insrc})=\underbrace{J_{\mathcal{M}_\text{src}}(\mu)-J_{\mathcal{M}_\text{src}}(\pi^\star_{\text{insrc}})}_{\text{(i)}}+\underbrace{J_{\mathcal{M}_\text{src}}(\hat\pi)-J_{\mathcal{M}_\text{src}}(\mu)}_{\text{(ii)}}.
\end{align*}

We first focus on term (i). By using performance difference lemma~\citep{kakade2002approximately}, the return difference between $\pi$ and $\pi^\star_\text{insrc}$ in $\mathcal{M}_\text{src}$ gives:
\begin{equation}
\label{eq:proof6}
\begin{aligned}
    J_{\mathcal{M}_\text{src}}(\mu)-J_{\mathcal{M}_\text{src}}(\pi^\star_\text{insrc})&=\int_{s}d_{\mu}(s)\int_{a}\left[\mu(a|s)A_{\pi^\star_\text{insrc}}(s,a)\right]\\
    &=\mathbb{E}_{s\sim d_\mu(\cdot),a\sim\mu(\cdot|s)}\left[A_{\pi^\star_\text{insrc}}(s,a)\right].
\end{aligned}
\end{equation}

Thus we get the first term in the desired bound. Then we turn to derive the second term.

Based on the Corollary \textcolor{blue}{1} in~\citep{achiam2017constrained}, we have
\begin{align*}
    J_{\mathcal{M}_\text{src}}(\hat\pi)-J_{\mathcal{M}_\text{src}}(\mu)\geq\int_{s}d_\mu(s)\int_a\left[\hat\pi(a|s)A_\mu(s,a)\right]-\frac{2\gamma\epsilon_{\mu}^{\hat\pi}}{(1-\gamma)^2}D^{d_\mu}_{TV}(\hat\pi,\mu),
\end{align*}
where $\epsilon_\mu^{\hat\pi}=\max_{s}\left[\mathbb{E}_{a\sim\hat\pi}A_\mu(s,a)\right]$, and $D^{d_\mu}_{TV}(\hat\pi,\mu)=\frac{1}{2}\int_{s}d_\mu(s)\int_{a}\left|\hat\pi(a|s)-\mu(a|s)\right|$ is the total variance distance between $\hat\pi$ and $\mu$ over the distribution $d_\mu$.

Note that under the assumption that for all $(s,a)$, then $\left(\hat\pi(a|s)-\mu(a|s)\right)A_\mu(s,a)\geq 0$, we have
\begin{equation}
\label{eq:proof5}
    \begin{aligned}
        &\int_{s}d_\mu(s)\int_a\left[\hat\pi(a|s)A_\mu(s,a)\right]\\
        &=\int_{s}d_\mu(s)\int_a\left[\left(\hat\pi(a|s)-\mu(a|s)\right)A_\mu(s,a)\right] + \int_{s}d_\mu(s)\int_a\left[\mu(a|s)A_\mu(s,a)\right] \\
        &\geq0+0\\
        &=0.
    \end{aligned}
\end{equation}

Equation~\ref{eq:proof5} uses the fact that $\int_a\left[\mu(a|s)A_\mu(s,a)\right]=0$. An important fact is that $\pi$ is updated via offline RL (such as IQL), which imposes implicit or explicit policy constraints on $\pi$. We follow the constraints in IQL and assume
\begin{align*}
    \max_{s}(\text{KL}(\mu(\cdot|s),\hat\pi(\cdot|s)),\text{KL}(\hat\pi(\cdot|s),\mu(\cdot|s)))\leq\epsilon.
\end{align*}

Then we have
\begin{align*}
    \begin{aligned}
        J_{\mathcal{M}_\text{src}}(\hat\pi)-J_{\mathcal{M}_\text{src}}(\mu)&\geq 0-\frac{2\gamma\epsilon_{\mu}^{\hat\pi}}{(1-\gamma)^2}D^{d_\mu}_{TV}(\hat\pi,\mu)\\
        &\geq -\frac{2\gamma\epsilon_{\mu}^{\hat\pi}}{(1-\gamma)^2}\int_sd_\mu(s)\min\left(\sqrt{\text{KL}(\mu(\cdot|s),\hat\pi(\cdot|s))},\sqrt{\text{KL}(\hat\pi(\cdot|s),\mu(\cdot|s))} \right)\\
        &\geq -\frac{2\gamma\epsilon_{\mu}^{\hat\pi}}{(1-\gamma)^2}\int_sd_\mu(s)\max\left(\sqrt{\text{KL}(\mu(\cdot|s),\hat\pi(\cdot|s))},\sqrt{\text{KL}(\hat\pi(\cdot|s),\mu(\cdot|s))} \right)\\
        &\geq -\frac{2\gamma\epsilon_{\mu}^{\hat\pi}}{(1-\gamma)^2}\cdot\sqrt{\epsilon}.
    \end{aligned}
\end{align*}

The second inequality results from Pinsker's inequality~\citep{csiszar2011information}. Hence, we conclude that for policy $\hat\pi$ learned on the source domain via IQL, it induces a safe policy improvement:
\begin{equation}
\label{eq:proof7}
    J_{\mathcal{M}_\text{src}}(\hat\pi)-J_{\mathcal{M}_\text{src}}(\mu)\geq-\mathcal{O}\left(\frac{1}{(1-\gamma)^2}\right).
\end{equation}

By combining the result of Equation~\ref{eq:proof6} and Equation~\ref{eq:proof7}, we have
\begin{equation}
    J_{\mathcal{M}_\mathrm{src}}(\hat{\pi})-J_{\mathcal{M}_\mathrm{src}}(\pi^\star_\mathrm{insrc})\geq \mathbb{E}_{s\sim\rho_\mu(\cdot),a\sim\mu(\cdot|s)}\left[A_{\pi^\star_{\mathrm{insrc}}}(s,a)\right] - \mathcal{O}\left(\frac{1}{(1-\gamma)^2}\right).
\end{equation}
Since for a policy $\hat\pi$ learned using an in-sample offline RL algorithm (e.g., IQL), $J_{\mathcal{M}_\mathrm{src}}(\hat{\pi})-J_{\mathcal{M}_\mathrm{src}}(\pi^\star_\mathrm{insrc})\leq0$ holds. Therefore we have
\begin{equation}
    \left|J_{\mathcal{M}_\mathrm{src}}(\hat{\pi})-J_{\mathcal{M}_\mathrm{src}}(\pi^\star_\mathrm{insrc})\right|\leq -\mathbb{E}_{s\sim\rho_\mu(\cdot),a\sim\mu(\cdot|s)}\left[A_{\pi^\star_{\mathrm{insrc}}}(s,a)\right] + \mathcal{O}\left(\frac{1}{(1-\gamma)^2}\right).
\end{equation}
Then we conclude the proof.    
\end{proof}

\subsection{Proof of Proposition~\ref{proposition:3}}



\begin{proof}
We first decompose the objective into two terms:
\begin{align*}
    \begin{aligned}
        \mathbb{E}_{s\sim\rho_\mu(\cdot),a\sim\mu(\cdot|s)}\left[\hat{A}_\text{pre}(s,a)-A_{\pi^\star_\text{insrc}}(s,a)\right]&=\mathbb{E}_{s\sim\rho_\mu(\cdot),a\sim\mu(\cdot|s)}\left[A_\text{pre}(s,a)-A_{\pi^\star_\text{insrc}}(s,a)\right]\\
        &+\mathbb{E}_{s\sim\rho_\mu(\cdot),a\sim\mu(\cdot|s)}\left[\hat{A}_\text{pre}(s,a)-A_\text{pre}(s,a)\right].
    \end{aligned}
\end{align*}

For the first term, using the performance difference lemma~\citep{kakade2002approximately}, we have
    \begin{equation}
    \label{eq:proof8}
        J_{\mathcal{M}_\text{src}}(\mu)-J_{\mathcal{M}_\text{src}}(\pi_\text{pre})=\mathbb{E}_{s\sim\rho_{\mu}(\cdot),a\sim\mu(\cdot|s)}\left[A_{\pi_\text{pre}}(s,a)\right],
    \end{equation}
    \begin{equation}
    \label{eq:proof9}
        J_{\mathcal{M}_\text{src}}(\mu)-J_{\mathcal{M}_\text{src}}(\pi^\star_\text{insrc})=\mathbb{E}_{s\sim\rho_{\mu}(\cdot),a\sim\mu(\cdot|s)}\left[A_{\pi^\star_\text{insrc}}(s,a)\right].
    \end{equation}

By subtracting Equation~\ref{eq:proof8} with Equation~\ref{eq:proof9}, we can get $\Delta J_{\mathcal{M}_\mathrm{src}}(\pi^\star_\text{insrc},\pi_\text{pre})$. The second term is exactly $\mathbb{E}_{s\sim\rho_\mu(\cdot),a\sim\mu(\cdot|s)}\left[\Delta(s,a)\right]$. This concludes the proof.

\end{proof}

\section{Environment Setting}

In this section, we supplement the detailed environmental settings we adopt in our experiments, including the information of source and target domain datasets, and the code-level realization of kinematic shifts and morphology shifts, etc.

\subsection{Datasets and Metrics}

\begin{figure}[t]
    \centering
    \includegraphics[width=1.0\linewidth]{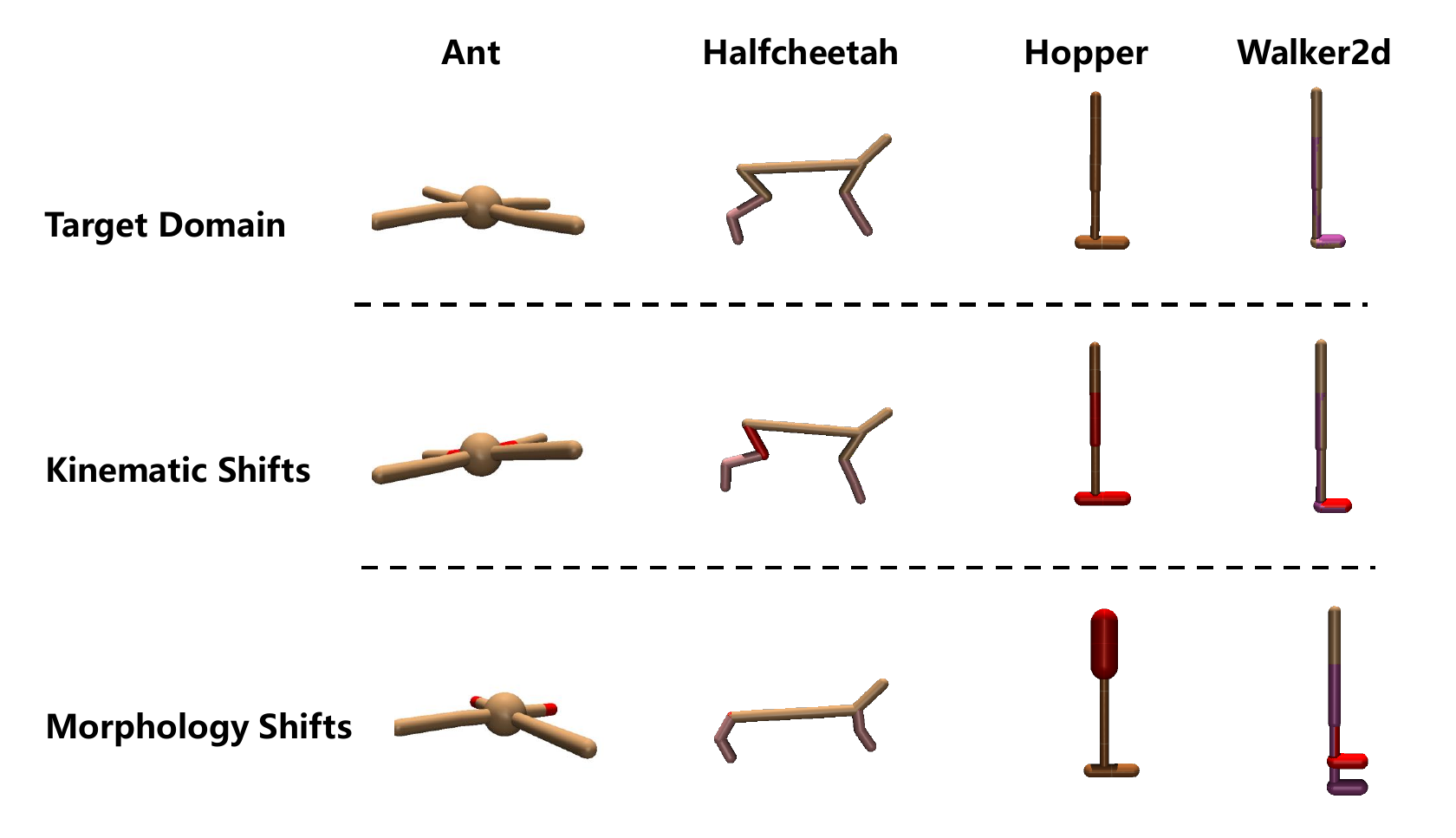}
    \caption{Visualization of the target domains and source domains with kinematic shifts and morphology shifts, across four tasks (\texttt{ant}, \texttt{halfcheetah}, \texttt{hopper}, \texttt{walker2d}).} 
    \vspace{-0.4cm}
    \label{fig:visualization}
\end{figure}

\textbf{Target domain datasets.} In the cross-domain offline RL setting, the target domain datasets should be collected in the target environment, and only limited target domain data is available. To this end, we directly sample a proportion of data from widely used D4RL~\citep{fu2020d4rl} MuJoCo datasets as the target domain datasets. In D4RL, the MuJoCo datasets are collected with an online SAC~\citep{haarnoja2018soft} agent in the environments of Gym~\citep{brockman2016openai} simulated by MuJoCo engine~\citep{todorov2012mujoco}. We adopt four tasks in our experiments: \texttt{halfcheetah-v2}, \texttt{hopper-v2}, \texttt{walker2d-v2}, \texttt{ant-v3}, and consider five dataset qualities for each task: \texttt{random}, \texttt{medium}, \texttt{medium-replay}, \texttt{medium-expert}, \texttt{expert}. In D4RL, the \textbf{random} datasets are collected with a randomly initialized policy. The \textbf{medium} datasets contain 1M samples collected from an early-stopped SAC policy. The \textbf{medium-replay} datasets record the replay buffer of an SAC agent trained up to the performance of the medium-level agent. The \textbf{medium-expert} datasets are a mixture of medium data and expert data at a 50-50 ratio. The \textbf{expert} datasets contain 1M samples logged from an expert policy. To construct target domain datasets with different sizes, we sample different numbers of data from the D4RL datasets. In Section~\ref{sec:main}, we sample $10\%$ data from each D4RL dataset, and in Section~\ref{appendix:extended}, we only sample 5,000 transitions for each dataset to simulate a very limited target data setting.

\textbf{Source domain datasets.} To fully examine the effectiveness of our method, we design two kinds of dynamics shift scenarios based on four widely used MuJoCo tasks: \texttt{halfcheetah-v2}, \texttt{hopper-v2}, \texttt{walker2d-v2}, \texttt{ant-v3}. The types of dynamics shift we implement include \texttt{kinematic shift} and \texttt{morphology shift}. The kinematic shift means that some joints of the simulated robot are broken and fail to rotate. The morphology shift indicates that the morphology of the simulated robot in the two domains is different. We show the visualization results of the simulated robot in the source domain and target domain in Figure~\ref{fig:visualization}. And the code-level modifications for realizing the dynamics shifts are deferred to the following subsections.

To construct source domain datasets, we follow a similar data-collecting procedure as in D4RL and collect the datasets in the revised environments. We train an online SAC agent within the environments with different kinds of dynamics shifts for 1M steps, and we log the checkpoints of the policy with different training steps and use them to roll out trajectories. The \textbf{random} datasets are generated by directly sampling the action space. The \textbf{medium} datasets are gathered with the logged policy that exhibits about $1/2$ performance of the expert policy. The \textbf{medium-replay} datasets consist of the replay buffer of the medium-level agent. We sample $50\%$ data from the medium datasets and $50\%$ data from the expert datasets, then we mix the sampled data and construct the \textbf{medium-expert} datasets. The \textbf{expert} datasets are gathered using the last policy checkpoint.

\textbf{Metrics.} The metric we use for evaluating the performance of the offline policy in the target domain is the \texttt{normalized score} (NS) in D4RL. It is computed as follows:
\begin{equation*}
    \text{NS}=\frac{J_{\pi}-J_{\text{random}}}{J_{\text{expert}}-J_{\text{random}}}\times 100\%,
\end{equation*}
where $J$ is the return acquired by the agent in the target domain, $J_{\text{random}}$ and $J_{\text{expert}}$ are the returns obtained by the random policy and the expert policy in the target domain, respectively.

\subsection{Kinematic Shift Tasks}
\label{appendix:realizationkinematic}
To simulate the kinematic shifts between the source domain and target domain, we modify the \texttt{xml} files of the original environments. Specifically, we change the rotation angle of some joints of the simulated robot for different tasks:

\textbf{\textit{halfcheetah-kinematic}:} The rotation angle of the joint on the thigh of the robot's back leg is modified from $[-0.52,1.05]$ to $[-0.0052,0.0105]$.

\begin{lstlisting}[language=Python]
# broken back thigh joint
<joint axis="0 1 0" damping="6" name="bthigh" pos="0 0 0" range="-.0052 .0105" stiffness="240" type="hinge"/>
\end{lstlisting}

\textbf{\textit{hopper-kinematic}:} The rotation angle of the head joint is modified from $[-150,0]$ to $[-0.15,0]$ and the rotation angle of the foot joint is modified from $[-45,45]$ to $[-18,18]$.

\begin{lstlisting}[language=Python]
# broken head joint
<joint axis="0 -1 0" name="thigh_joint" pos="0 0 1.05" range="-0.15 0" type="hinge"/>
# broken foot joint
<joint axis="0 -1 0" name="foot_joint" pos="0 0 0.1" range="-18 18" type="hinge"/>
\end{lstlisting}

\textbf{\textit{walker2d-kinematic}:} The rotation angle of the right foot joint is modified from $[-45,45]$ to $[-0.45,0.45]$.

\begin{lstlisting}[language=Python]
# broken right foot joint
<joint axis="0 -1 0" name="foot_joint" pos="0 0 0.1" range="-0.45 0.45" type="hinge"/>
\end{lstlisting}

\textbf{\textit{ant-kinematic}:} The rotation angles of the joints on the hip of two front legs are modified from $[-30,30]$ to $[-0.3,0.3]$.

\begin{lstlisting}[language=Python]
# broken hip joints of front legs
<joint axis="0 0 1" name="hip_1" pos="0.0 0.0 0.0" range="-0.3 0.3" type="hinge"/>
<joint axis="0 0 1" name="hip_2" pos="0.0 0.0 0.0" range="-0.3 0.3" type="hinge"/>
\end{lstlisting}

\subsection{Morphology Shift Tasks}
\label{appendix:realizationmorph}
Akin to the kinematic shifts, we modify the morphology of the simulated robot to simulate the morphology shifts:

\textbf{\textit{halfcheetah-morph}:} The sizes of the back thigh and the forward thigh are modified.

\begin{lstlisting}[language=Python]
# back thigh
<geom fromto="0 0 0 -0.0001 0 -0.0001" name="bthigh" size="0.046" type="capsule"/>
<body name="bshin" pos="-0.0001 0 -0.0001">
# front thigh
<geom fromto="0 0 0 0.0001 0 0.0001" name="fthigh" size="0.046" type="capsule"/>
<body name="fshin" pos="0.0001 0 0.0001">
\end{lstlisting}

\textbf{\textit{hopper-morph}:} The head size of the robot is modified.

\begin{lstlisting}[language=Python]
# head size
<geom friction="0.9" fromto="0 0 1.45 0 0 1.05" name="torso_geom" size="0.125" type="capsule"/>
\end{lstlisting}

\textbf{\textit{walker2d-morph}:} The thigh on the right leg of the robot is modified.

\begin{lstlisting}[language=Python]
# right leg
<body name="thigh" pos="0 0 1.05">
<joint axis="0 -1 0" name="thigh_joint" pos="0 0 1.05" range="-150 0" type="hinge"/>
<geom friction="0.9" fromto="0 0 1.05 0 0 1.045" name="thigh_geom" size="0.05" type="capsule"/>
<body name="leg" pos="0 0 0.35">
  <joint axis="0 -1 0" name="leg_joint" pos="0 0 1.045" range="-150 0" type="hinge"/>
  <geom friction="0.9" fromto="0 0 1.045 0 0 0.3" name="leg_geom" size="0.04" type="capsule"/>
  <body name="foot" pos="0.2 0 0">
    <joint axis="0 -1 0" name="foot_joint" pos="0 0 0.3" range="-45 45" type="hinge"/>
    <geom friction="0.9" fromto="-0.0 0 0.3 0.2 0 0.3" name="foot_geom" size="0.06" type="capsule"/>
  </body>
</body>
</body>
\end{lstlisting}

\textbf{\textit{ant-morph}:} The size of the robot's two front legs is reduced.

\begin{lstlisting}[language=Python]
# front leg 1
<geom fromto="0.0 0.0 0.0 0.1 0.1 0.0" name="left_ankle_geom" size="0.08" type="capsule"/>
# front leg 2
<geom fromto="0.0 0.0 0.0 -0.1 0.1 0.0" name="right_ankle_geom" size="0.08" type="capsule"/>
\end{lstlisting}

\section{Implementation Details}
\label{appendix:implementation}
In this section, we provide more details about the implementation of the baseline methods and our method. We also list the hyperparameter setup for all methods.

\subsection{Baselines}
\label{appendix:baseline}
\textbf{IQL:} IQL~\citep{kostrikov2021offline} is an off-the-shelf offline RL algorithm that learns the policy in an \textit{in-sample manner}, which means no OOD samples that lie outside of the offline datasets are required during training. In the cross-domain offline setting, we follow the algorithm procedure but draw samples from both the source domain dataset and the target domain dataset. IQL trains the state value function via expectile regression:
\begin{align*}
    \mathcal{L}_{V}=\mathbb{E}_{(s,a)\sim \mathcal{D}_{\text{src}}\cup\mathcal{D}_{\text{tar}}}\left[L_2^\tau(Q_{\theta^\prime}(s,a)-V_{\psi}(s))\right]
\end{align*}
where $L_2^\tau(u)=\left|\tau-\mathbb{I}(u<0)\right|u^2$, $\mathbb{I}(\cdot)$ is the indicator function, and $\theta^\prime$ is the target network parameter. With such expectile regression, an in-sample optimal value function can be learned. Then the state-action value function is updated by:
\begin{align*}
    \mathcal{L}_{Q}=\mathbb{E}_{(s,a,r,s^\prime)\sim\mathcal{D}_{\text{src}}\cup\mathcal{D}_{\text{tar}}}\left[(r(s,a)+\gamma V_\psi(s^\prime)-Q_\theta(s,a))^2\right]
\end{align*}
Then the advantage value is calculated as $A(s,a)=Q(s,a)-V(s,a)$ and the policy is extracted by advantage weighted behavior cloning:
\begin{align*}
    \mathcal{L}_{\pi}=-\mathbb{E}_{(s,a)\sim \mathcal{D}_{\text{src}}\cup\mathcal{D}_{\text{tar}}}\left[\exp(\beta\times A(s,a))\log\pi_{\phi}(a|s)\right],
\end{align*}
where $\beta$ is the inverse temperature coefficient. We implement IQL by following its offlicial codebase\footnote{\textcolor{purple}{https://github.com/ikostrikov/implicit\_q\_learning.git}}.

\textbf{BOSA:} BOSA~\citep{liu2024beyond} defines the issues of the state-action OOD problem and the dynamics OOD problem in cross-domain offline RL, and proposes two support constraints to tackle the issues. To be specific, BOSA handles the OOD state-action problem by supported policy optimization, and mitigates the OOD dynamics problem by supported value optimization. BOSA updates the critic by supported value optimization:
\begin{align*}
    \mathcal{L}_Q=\mathbb{E}_{(s,a)\sim\mathcal{D}_{\text{src}}}\left[Q_{\theta_i}(s,a)\right] + \mathbb{E}_{\substack{(s,a,r,s^\prime)\sim\mathcal{D}_{\text{src}}\cup\mathcal{D}_\text{tar},\\ a^\prime\sim\pi_\phi(s^\prime)}}\left[\mathbb{I}(\hat{P}_\text{tar}(s'|s,a) > \epsilon)(Q_{\theta_i}(s,a) - y)^2\right],
\end{align*}
where $\mathbb{I}(\cdot)$ is the indicator function, and $\hat{P}_{\text{tar}}(s^\prime|s,a)$ is the target domain transition dynamics estimated via maximum likelihood estimation, and $\epsilon$ is the threshold coefficient. The policy in BOSA is updated by supported policy optimization:
\begin{align*}
    \mathcal{L}_\pi=\mathbb{E}_{s \sim \mathcal{D}_{\text{src}} \cup \mathcal{D}_{\text{tar}}, \ a \sim \pi_{\phi}(s)} \left[Q_{\theta_i}(s, a)\right], \quad \text{s.t. } \mathbb{E}_{s \sim \mathcal{D}_{\text{src}} \cup \mathcal{D}_{\text{tar}}} \left[\hat{\pi}_{\text{mix}}(\pi_{\phi}(s) \mid s)\right] > \epsilon',
\end{align*}
where $\epsilon^\prime$ is the threshold coefficient, and $\hat{\pi}_{\phi_{\text{mix}}}(\cdot|s)$ is the empirical behavior policy of the mixed datasets $\mathcal{D}_{\text{src}}\cup\mathcal{D}_{\text{tar}}$ learned with CVAE~\citep{kingma2013auto}. We do not find the official implementation for BOSA, so we use the codebase\footnote{\textcolor{purple}{https://github.com/OffDynamicsRL/off-dynamics-rl.git}} in ODRL benchmark~\citep{lyu2024odrlabenchmark}, which provides high-quality implementations for various off-dynamics RL algorithms.

\textbf{DARA.} DARA~\citep{liu2022dara} leverages dynamics-aware reward modification to fulfill dynamics adaptation and is the offline version of DARC~\citep{eysenbach2020off}. DARA trains two domain classifiers $q_{\theta_{SAS}}(\text{target}|s_t,a_t,s_{t+1})$ and $q_{\theta_{SA}}(\text{target}|s_t,a_t)$ as follows.
\begin{align*}
    \begin{aligned}
    \mathcal{L}_{\theta_{SAS}}&=\mathbb{E}_{\mathcal{D}_{\text{tar}}}\left[\log q_{\theta_{SAS}}(\text{target}|s_t, a_t, s_{t+1})\right]+\mathbb{E}_{\mathcal{D}_{\text{src}}}\left[\log(1- q_{\theta_{SAS}}(\text{target}|s_t, a_t, s_{t+1}))\right], \\
    \mathcal{L_{\theta_{SA}}}&=\mathbb{E}_{\mathcal{D}_{\text{tar}}}\left[\log q_{\theta_{SA}}(\text{target}|s_t, a_t)\right]+\mathbb{E}_{\mathcal{D}_{\text{src}}}\left[\log(1- q_{\theta_{SA}}(\text{target}|s_t, a_t))\right],
    \end{aligned}
\end{align*}
The two domain classifiers are used to estimate the dynamics gap $\log\frac{P_{\mathcal{M}_{\text{tar}}}(s_{t+1}|s_t,a_t)}{P_{\mathcal{M}_{\text{src}}}(s_{t+1}|s_t,a_t)}$ between the source domain and the target domain. Then the estimated dynamics gap is used as a penalty to the source domain rewards:
\begin{align*}
    \hat{r}_{\text{DARA}} = r - \lambda \times \delta_r, \quad \delta_r(s_t, a_t) = - \log \frac{q_{\theta_{\text{SAS}}}(\text{target} | s_t, a_t, s_{t+1}) q_{\theta_{\text{SA}}}(\text{source} | s_t, a_t)}{q_{\theta_{\text{SAS}}}(\text{source} | s_t, a_t, s_{t+1}) q_{\theta_{\text{SA}}}(\text{target} | s_t, a_t)},
\end{align*}
where $\lambda$ controls the intensity of the reward penalty. We use the re-implementation in ODRL for DARA. $\lambda$ is set to $0.1$, and the reward penalty is clipped within $[-10,10]$ for training stability.

\textbf{IGDF.} IGDF~\citep{wen2024contrastive} estimates the domain gap between the source domain and the target domain with contrastive representation learning, and employs data filtering to share source domain samples with a smaller dynamics gap for training. IGDF trains a score function $h(\cdot)$ using $(s,a,s^\prime_{\text{tar}})\sim\mathcal{D}_{\text{tar}}$ as the positive samples, and transitions $(s,a,s^\prime_{\text{src}})$ as the negative samples, where $(s,a)\sim\mathcal{D}_{\text{tar}}$ and $s^\prime_{\text{src}}\sim\mathcal{D}_{\text{src}}$. $h(\cdot)$ is trained with the contrastive learning objective:
\begin{equation}
\label{eq:contrastive}
    \mathcal{L}=-\mathbb{E}_{(s,a,s^\prime_{\text{tar}})}\mathbb{E}_{s^\prime_{\text{src}}}\left[\log\frac{h(s,a,s^\prime_{\text{tar}})}{\sum_{s^\prime\in s^\prime_{\text{tar}}\cup s^\prime_{\text{src}}}h(s,a,s^\prime)}\right].
\end{equation}
For the construction of the score function, IGDF adopts two networks $\phi(s,a)$ and $\psi(s^\prime)$ to learn the representations of state-action and state, respectively. The score function is expressed as a linear parameterization of $\phi(s,a)$ and $\psi(s^\prime)$:
\begin{align*}
    h(s,a,s^\prime)=\exp(\phi(s,a)^T\psi(s^\prime)).
\end{align*}
Based on the learned score function, IGDF proposes to selectively share source domain data for training value functions:
\begin{align*}
    \mathcal{L}_Q = \frac{1}{2} \mathbb{E}_{\mathcal{D}_{\text{tar}}} \left[ (Q_{\theta} - \mathcal{T}Q_{\theta})^2 \right] + \frac{1}{2} \alpha \cdot h(s, a, s') \mathbb{E}_{(s, a, s') \sim \mathcal{D}_{\text{src}}} \left[ \mathbb{I}(h(s, a, s') > h_{\xi\%})(Q_{\theta} - \mathcal{T}Q_{\theta})^2 \right],
\end{align*}
where $\mathbb{I}(\cdot)$ is the indicator function, $\alpha$ is the weighting coefficient, $\xi$ is the data selection ratio. We implement IGDF by following its official codebase\footnote{\textcolor{purple}{https://github.com/BattleWen/IGDF.git}}.

\textbf{OTDF.} OTDF~\citep{lyu2025cross} depicts the distance between the source domain data and target domain data by computing the Wasserstein distance~\citep{peyre2019computational}:
\begin{equation}
\label{eq:ot}
\mathcal{W}(u, u') = \min_{\mu \in M} \sum_{t=1}^{|\mathcal{D}_{\text{src}}|} \sum_{t'=1}^{|\mathcal{D}_{\text{tar}}|} C(u_t, u'_{t'}) \cdot \mu_{t, t'},
\end{equation}
where $u=s_{\text{src}}\oplus a_{\text{src}} \oplus s'_{\text{src}}$, $u^\prime=s_{\text{tar}}\oplus a_{\text{tar}} \oplus s'_{\text{tar}}$, $C$ is the cost function and $M$ is the coupling matrices. After solving Equation~\ref{eq:ot} for $\mu^\star$, the OTDF determines the deviation between a source domain dataset and the target domain dataset via:
\begin{align*}
    d(u_t)=-\sum_{t^\prime=1}^{\left|\mathcal{D}_{\text{tar}}\right|}C(u_t,u_{t^\prime})\mu_{t,t^\prime}^\star,\quad u_t=(s_{\text{src}}^t,a^t_{\text{src}},(s^\prime_{\text{src}})^t)\sim\mathcal{D}_{\text{src}}.
\end{align*}
Then the critic is updated by
\begin{align*}
    \mathcal{L}_Q = \mathbb{E}_{\mathcal{D}_{\text{tar}}} \left[ (Q_\theta - \mathcal{T}Q_\theta)^2 \right] + \mathbb{E}_{(s, a, s') \sim \mathcal{D}_{\text{src}}} \left[ \exp(\alpha \times d) \mathbb{I}(d > d_{\%}) (Q_\theta - \mathcal{T}Q_\theta)^2 \right].
\end{align*}
Besides, OTDF includes an extra policy regularization term that encourages the policy to be close to the support region of the target dataset:
\begin{align*}
    \widehat{\mathcal{L}_\pi}=\mathcal{L}_\pi-\beta\times\mathbb{E}_{s\sim\mathcal{D}_{\text{src}}\cup\mathcal{D}_{\text{tar}}}\log\pi_{\text{tar}}^b(\pi(\cdot|s)|s),
\end{align*}
where $\mathcal{L}_\pi$ is the original policy optimization objective and $\beta$ is the weight coefficient. We run the official code\footnote{\textcolor{purple}{https://github.com/dmksjfl/OTDF.git}} for OTDF in our experiments.

\subsection{More Details of Motivation Example}
\label{appendix:motivation}

In this section, we supplement with more details for our motivation example in Section~\ref{sec:motivation}. 

\begin{figure}[t]
    \centering
    \includegraphics[width=0.98\linewidth]{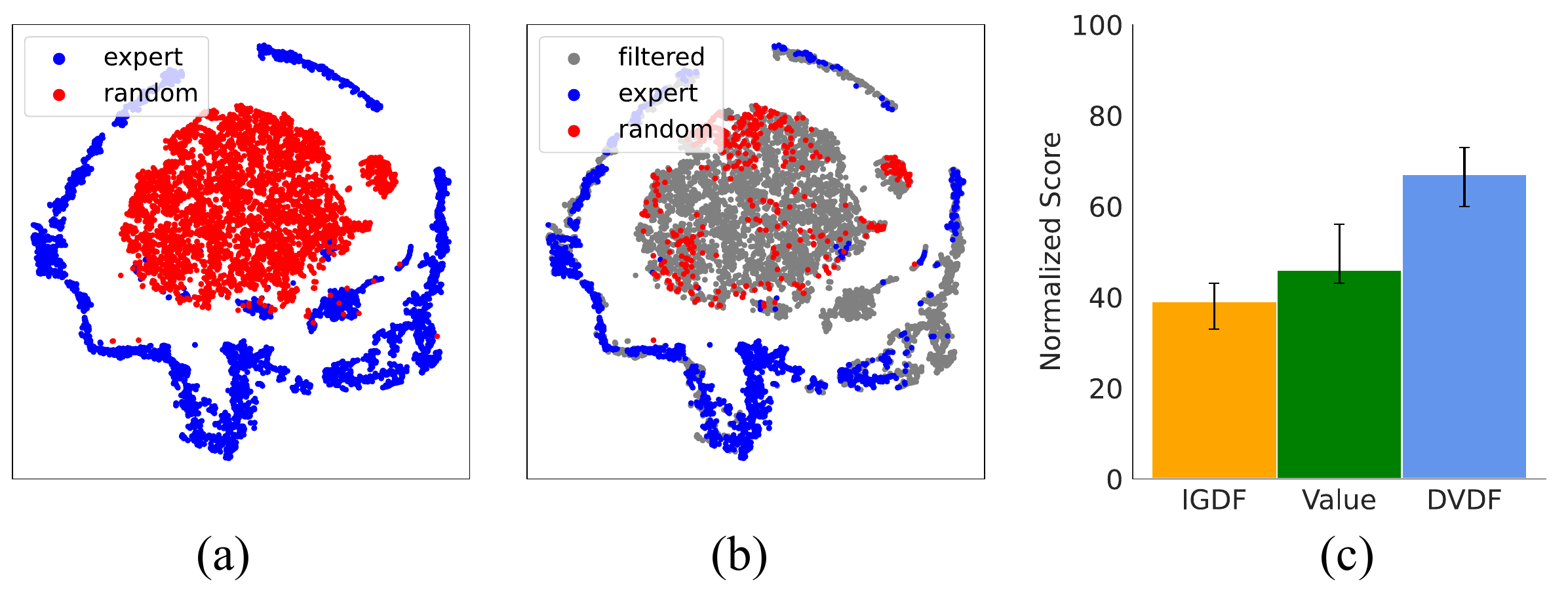}
    \caption{\textbf{(a)} Visualization of source domain data. \textbf{(b)} Source domain data filtering visualization of Value-IGDF. \textbf{(c)} Performance comparison between IGDF, Value-IGDF and DVDF.} 
    \vspace{-0.4cm}
    \label{fig:motivation_appendix}
\end{figure}

We provide the visualization results of the random-expert mixed source domain dataset in Figure~\ref{fig:motivation_appendix} (a). We further conduct an experiment to demonstrate the necessity of dynamics- and value-aligned data filtering. Instead of using $h(s,a,s^\prime)$ or $g(s,a,s^\prime)$ as the indicator like IGDF and DVDF, we directly use $A_\mathrm{pre}(s,a)$ for data filtering, which means we select source domain data with a smaller value misalignment, disregarding dynamics misalignment. We term this modified algorithm version as Value-IGDF. We visualize the data filtering results (data selection ratio $\xi$ is $25\%$) in Figure~\ref{fig:motivation_appendix} (b), which indicates that the selected samples are predominantly expert samples, despite the dynamics shifts. We evaluate the performance of IGDF, Value-IGDF, and DVDF on the source domain and present the results in Figure~\ref{fig:motivation_appendix} (c). We can see that while Value-IGDF outperforms IGDF, it still lags behind DVDF, highlighting the necessity of jointly considering dynamics and value alignment.

\subsection{Algorithmic Details of DVDF}
\label{appendix:pseudocode}

As a plug-in module, DVDF can be seamlessly integrated into various cross-domain offline RL algorithms, such as IGDF and OTDF. In this section, we present more details about the combination of DVDF with IGDF (tagged DVDF-IGDF) and OTDF (namely DVDF-OTDF), and summarize the pseudocodes of DVDF-IGDF and DVDF-OTDF.

For DVDF-IGDF, the new score function $g(\cdot)$ incorporates the vanilla score function $h(\cdot)$ and the pre-trained advantage function $A_{\text{pre}}(\cdot)$. During training, only $h(\cdot)$ is updated and $A_{\text{pre}}(\cdot)$ remains frozen. We leverage $g(\cdot)$ for data filtering, and the other procedures of DVDF-IGDF remain identical to those of IGDF. The detailed pseudocode of DVDF-IGDF is presented in Algorithm~\ref{alg:dv_igdf}. The \textcolor{blue}{blue} texts mark the different algorithm procedure from the original IGDF.

\begin{algorithm}[htbp]
\caption{DVDF-IGDF}
\label{alg:dv_igdf}
\begin{algorithmic}[1]
\STATE \textbf{Require:} Source domain offline dataset $\mathcal{D}_{\text{src}}$, target domain offline dataset $\mathcal{D}_{\text{tar}}$, mixed offline dataset $\mathcal{D}_{\text{mix}}$
\STATE \textbf{Initialization:} Policy network $\pi_\eta$, value network $V_\beta$, target $Q$ network $Q_\theta$, encoder networks $\phi(s, a)$, $\psi(s')$, data selection ratio $\xi$, batch size $B$, importance coefficient $\alpha$, alignment tradeoff coefficient $\lambda$
\STATE \textbf{// Pre-train the advantage function}
\STATE \textcolor{blue}{Pre-train an SQL agent on $\mathcal{D}_\text{src}$, obtain $\hat{A}_\text{pre}(\cdot)$ and normalize $\hat{A}_\text{pre}(\cdot)$}
\STATE \textbf{// Contrastive Representation Learning}
\STATE Train $h(s,a,s^\prime)$ via contrastive representation learning by Equation~\ref{eq:contrastive}, \textcolor{blue}{obtain the score function $\omega(s,a,s^\prime)=\lambda\cdot h(s,a,s^\prime)+(1-\lambda)\cdot \hat{A}_\text{pre}(s,a)$}
\STATE \textbf{// TD Learning}
\FOR{each gradient step}
    \STATE Sample $b_{\text{src}} := \{(s, a, r, s')\}$ from $\mathcal{D}_{\text{src}}$
    \STATE Sample $b_{\text{tar}} := \{(s, a, r, s')\}$ from $\mathcal{D}_{\text{tar}}$
    \STATE \textcolor{blue}{Sample the top-$\xi$ samples from $b_{\text{src}}$ ranked by $g(s_{\text{src}}, a_{\text{src}},s^\prime_\text{src})$}
    \STATE Compute weights $\omega(s, a, s')$ following:
    \STATE \quad \textcolor{blue}{$\omega(s, a, s') = \mathbb{I}(g(s,a,s^\prime) \geq g_{\xi\%})$}
    \STATE \textbf{// Optimize the $V_\beta$ function}
    \STATE Compute loss $\mathcal{L}_V$:
    \STATE \quad $\mathcal{L}_V = \mathbb{E}_{(s, a) \sim \mathcal{D}_{\text{src}} \cup \mathcal{D}_{\text{tar}}} \left[ L_2^\tau\left( Q_\theta(s, a) - V_\beta(s) \right) \right]$
    \STATE Update $V_\beta$ using $\mathcal{L}_V$
    \STATE \textbf{// Optimize the $Q_\theta$ function}
    \STATE Compute loss $\mathcal{L}_Q$:
    \STATE \quad \textcolor{blue}{$\mathcal{L}_Q = \frac{1}{2}\cdot\mathbb{E}_{(s, a, r, s') \sim \mathcal{D}_{\text{tar}}} \left[ \left( Q_\theta(s, a) - (r + \gamma V_\beta(s')) \right)^2 \right]$}
    \STATE \quad \textcolor{blue}{$+ \frac{1}{2}\cdot\mathbb{E}_{(s, a, r, s') \sim \mathcal{D}_{\text{src}}} \left[ \omega(s, a, s')g(s,a,s^\prime) \left( Q_\theta(s, a) - (r + \gamma V_\beta(s')) \right)^2 \right]$}
    \STATE Update $Q_\theta$ using $\mathcal{L}_Q$
    \STATE \textbf{// Update target network}
    \STATE Update target network parameters: $\theta' \leftarrow (1 - \mu) \theta + \mu \theta'$
    \STATE \textbf{// Policy Extraction (AWR)}
    \STATE Compute advantage $A(s, a) = Q_\theta(s, a) - V_\beta(s)$
    \STATE Optimize policy network $\pi_\eta$ using advantage-weighted regression (AWR):
    \STATE \quad $\mathcal{L}_\pi = \mathbb{E}_{(s, a) \sim \mathcal{D}_{\text{src}} \cup \mathcal{D}_{\text{tar}}} \left[ \exp(\alpha A(s, a)) \log \pi_\eta(a|s) \right]$
\ENDFOR
\end{algorithmic}
\end{algorithm}

For DVDF-OTDF, other than $A_{\text{pre}}(\cdot)$, the optimal coupling $\mu^\star$ also needs to be computed before the policy training process begins. We choose cosine distance as the cost function and utilize OTT-JAX library~\citep{cuturi2022optimal} for solving the OT problem. Note that DVDF only plays a part in the procedure of data filtering, and the remaining process of training remains the same as that in OTDF. We summarize the pseudocode of DVDF-OTDF in Algorithm~\ref{alg:dv_otdf}. The \textcolor{blue}{blue} lines highlight the different procedure from the vanilla OTDF.

\begin{algorithm}[htbp]
\caption{DVDF-OTDF}
\label{alg:dv_otdf}
\begin{algorithmic}[1]
\STATE \textbf{Input:} Source domain dataset $\mathcal{D}_{\text{src}}$, target domain dataset $\mathcal{D}_{\text{tar}}$, batch size $N$, data selection ratio $\xi$, alignment tradeoff coefficient $\lambda$
\STATE Initialize policy network $\pi_\phi$, value networks $V_\psi$, $Q_\theta$, target $Q$ function $Q_{\theta'}$, the cost function $C$, policy coefficients $\beta$, number of sampled latent variables $M$, target update rate $\eta$
\STATE \textbf{// Pre-train the advantage function}
\STATE \textcolor{blue}{Pre-train an SQL agent on $\mathcal{D}_\text{src}$, obtain $\hat{A}_\text{pre}(\cdot)$ and normalize $\hat{A}_\text{pre}(\cdot)$}
\STATE \textbf{// Solve the OT problem}
\STATE Compute the optimal alignment between $\mathcal{D}_{\text{src}}$ and $\mathcal{D}_{\text{tar}}$ with Equation 4
\STATE Compute deviations $\{d_v\}_{v=1}^{|\mathcal{D}_{\text{src}}|}$ between the source domain data and $\mathcal{D}_{\text{tar}}$ with Equation 5
\STATE Normalize the deviations $d_v$ to obtain normalized deviations $d^\prime_v$
\STATE \textcolor{blue}{Compute the score function $\omega(s,a,s^\prime)=\lambda\cdot d^\prime_v(s,a,s^\prime) + (1-\lambda)\cdot \hat{A}_\text{pre}(s,a,s^\prime)$}
\STATE Concatenate $\mathcal{D}_{\text{src}}$ and $\{\omega_v\}_{v=1}^{|\mathcal{D}_{\text{src}}|}$ to get $\mathcal{D}_{\text{src}}' = \{(s_v, a_v, r_v, s_v', \omega_v)\}_{v=1}^{|\mathcal{D}_{\text{src}}|}$
\FOR{$i = 1, 2, \dots$}
    \STATE Sample a mini-batch $b_{\text{src}} := \{(s, a, r, s', w)\}_{v=1}^{N/2}$ from $\mathcal{D}_{\text{src}}'$
    \STATE Sample a mini-batch $b_{\text{tar}} := \{(s, a, r, s')\}_{v=1}^{N/2}$ from $\mathcal{D}_{\text{tar}}$
    \STATE Update the state value function $V_\psi$ via:
    \STATE \quad $\mathcal{L}_V = \mathbb{E}_{(s,a) \sim \mathcal{D}_{\text{src}} \cup \mathcal{D}_{\text{tar}}} \left[ L_2^\tau\left( Q_\theta(s, a) - V_\psi(s) \right) \right]$
    \STATE \textbf{// Data filtering}
    \STATE Rank the deviations of the sample source domain data and reject the lowest $\xi\%$ of them
    \STATE Compute the weights for the remaining source domain data by $\exp(\beta \omega_v')$
    \STATE Compute the target value via: $y = r + \gamma V_\psi(s')$
    \STATE Optimize the state-action value function $Q_\theta$ on $b_{\text{src}} \cup b_{\text{tar}}$ via:
    \STATE \quad \textcolor{blue}{$\mathcal{L}_Q = \mathbb{E}_{(s,a,r,s') \sim \mathcal{D}_{\text{tar}}} \left[ \left( Q_\theta(s, a) - y \right)^2 \right] + \mathbb{E}_{(s,a,d) \sim \mathcal{D}_{\text{src}}'} \left[ \exp(\beta \omega_v') \left( Q_\theta(s, a) - y \right)^2 \right]$}
    \STATE Update the target network via: $\theta' \leftarrow \eta \theta + (1 - \eta) \theta'$
    \STATE \textbf{// Dataset regularization}
    \STATE Compute the advantage $A$ and optimize the policy $\pi_\phi$ on $b_{\text{src}} \cup b_{\text{tar}}$ using advantage-weighted regression (AWR) and dataset regulation:
    \STATE \quad $\mathcal{L}_\pi = \mathbb{E}_{(s,a) \sim \mathcal{D}_{\text{src}} \cup \mathcal{D}_{\text{tar}}} \left[ \exp(\beta A) \log \pi_\phi(a|s) - \beta \cdot \text{KL}\left( \pi_\phi(\cdot|s) \| \pi_{\text{prior}}(\cdot|s) \right) \right]$
\ENDFOR
\end{algorithmic}
\end{algorithm}


\subsection{Hyperparameter Setup}

We present the main hyperparameter setup in our experiments for all the methods we use in Table~\ref{tab:hyperparameters}.

\begin{table}[htbp]
\centering
\caption{Hyperparameter setup for DVDF and baseline methods}
\label{tab:hyperparameters}
\begin{tabular}{ll}
\toprule
\textbf{Hyperparameter} & \textbf{Value} \\
\midrule
\multicolumn{2}{l}{\textbf{Shared}} \\
Actor network & (256, 256) \\
Critic network & (256, 256) \\
Learning rate & $3 \times 10^{-4}$ \\
Optimizer & Adam~\citep{kingma2014adam} \\
Discount factor & 0.99 \\
Nonlinearity & ReLU \\
Target update rate & $5 \times 10^{-3}$ \\
Source domain Batch size & 128 \\
Target domain Batch size & 128 \\
\midrule
\multicolumn{2}{l}{\textbf{IQL}} \\
Temperature coefficient & 0.2 \\
Maximum log std & 2 \\
Minimum log std & -20 \\
Inverse temperature parameter $\beta$ & 3.0 \\
Expectile parameter $\tau$ & 0.7 \\
\midrule
\multicolumn{2}{l}{\textbf{DARA}} \\
Temperature coefficient & 0.2 \\
Classifier network & (256, 256) \\
Reward penalty coefficient $\lambda$ & 0.1 \\
\midrule
\multicolumn{2}{l}{\textbf{BOSA}} \\
Temperature coefficient & 0.2 \\
Maximum log std & 2 \\
Minimum log std & -20 \\
Policy regularization coefficient $\lambda_{\text{policy}}$ & 0.1 \\
Transition coefficient $\lambda_{\text{transition}}$ & 0.1 \\
Threshold parameter $\rho$ & $\log(0.01)$ \\
Value weight $\sigma$ & 0.1 \\
CVAE ensemble size of the dynamics model & 5 \\
\midrule
\multicolumn{2}{l}{\textbf{IGDF}} \\
Representation dimension & \{16, 64\} \\
Contrastive encoder network & (256, 256) \\
Encoder pretrained steps & 7000 \\
Importance coefficient & 1.0 \\
Data selection ratio $\xi$ & 25\% \\
\midrule
\multicolumn{2}{l}{\textbf{OTDF}} \\
CVAE training steps & 10000 \\
Number of sampled latent variables $M$ & 10 \\
Cost function & cosine \\
Data filtering ratio $\xi$ & $20\%$  \\
\midrule
\multicolumn{2}{l}{\textbf{DVDF-IGDF}} \\
SQL pre-training steps & $1\times10^6$ \\
Trade-off coefficient $\lambda$ & 0.7\\
Data selection ratio $\xi$ & $50\%$\\
\midrule
\multicolumn{2}{l}{\textbf{DVDF-OTDF}} \\
SQL pre-training steps & $1\times10^6$ \\
Trade-off coefficient $\lambda$ & 0.7 \\
Data filtering ratio $\xi$ & $50\%$
 \\
\bottomrule
\end{tabular}
\end{table}

\section{Wider Experimental Results}

\subsection{Results under Morphology Shifts}
\label{appendix:morphresults}
In Section~\ref{sec:main}, we give the evaluation results of our methods with different sizes of target domain datasets under kinematic shifts. In this section, we supplement with more evaluation results under morphology shifts, with other experimental settings identical to Section~\ref{sec:main}. 

Table~\ref{tab:resultsappendix} presents the comparison results using $10\%$ data of the D4RL datasets as the target domain data. The results clearly show that DVDF enhances the performance of the base algorithms. Specifically, DVDF-IGDF achieves the highest total score among all the methods, outperforming IGDF by $\textbf{15.4\%}$ (1039.0 to 1198.7) and performing better in \textbf{16} out of 20 tasks, while DVDF-OTDF improves OTDF by $\textbf{11.0\%}$ (1042.1 to 1156.3) and excels in \textbf{14} out of 20 tasks.

\begin{table}[t]
    \centering
    \begingroup
    \setlength{\tabcolsep}{11pt}
    \caption{\textbf{Performance comparison under morphology shifts.} half=halfcheetah, hopp=hopper, walk=walker2d, r=random, m=medium, me=medium-expert, mr=medium-replay, e=expert. We report the normalized score evaluated in the target domain, and $\pm$ captures the standard deviation across 5 seeds. We \textbf{bold} the highest scores for each task.}
    \label{tab:resultsappendix}
    \begin{tabular}{l|ccc|cc|cc}
    \toprule
    \textbf{Dataset} & IQL & BOSA & DARA & IGDF & DVDF-IGDF & OTDF & DVDF-OTDF\\
    \midrule
    half-r & \textbf{6.7} & 2.2 & 2.9 & \textbf{4.9$\pm$0.3} & 4.8$\pm$0.1 & \textbf{2.2$\pm$0.2} & 2.0$\pm$0.1 \\
    half-m & 45.8 & 41.3 & 45.6 & 45.5$\pm$0.1 & \textbf{46.0$\pm$0.3} & \textbf{44.3$\pm$0.2} & 42.5$\pm$0.2 \\
    half-mr & 26.1 & 27.8 & 28.9 & 24.2$\pm$3.3 & \textbf{31.1$\pm$4.7} & 19.7$\pm$2.5 & \textbf{27.2$\pm$1.3} \\
    half-me & 63.0 & 44.4 & 59.2 & 50.2$\pm$ 3.4 & \textbf{61.9$\pm$4.9} & 42.9$\pm$3.6 & \textbf{53.8$\pm$4.9} \\
    half-e & 65.2 & 78.6 & 55.4 & 43.0$\pm$6.2 & \textbf{51.7$\pm$6.8} & 74.2$\pm$5.0 & \textbf{91.7$\pm$7.0} \\
    hopp-r & 4.7 & 1.4 & 4.8 & \textbf{4.8$\pm$0.2} & 4.7$\pm$0.1 & \textbf{2.4$\pm$0.1} & 1.4$\pm$0.1 \\
    hopp-m & 56.4 & 28.7 & 49.5 & \textbf{55.5$\pm$2.9} & 52.7$\pm$4.6 & 49.1$\pm$2.2 & \textbf{59.4$\pm$3.7} \\
    hopp-mr & 51.3 & 40.6 & 53.5 & 54.9$\pm$5.8 & \textbf{58.6$\pm$6.4} & 24.9$\pm$3.4 &\textbf{32.6$\pm$4.5} \\
    hopp-me & 35.8 & 20.2 & 38.2 & 43.3$\pm$3.6 & \textbf{61.2$\pm$4.2} & 51.8$\pm$3.9 & \textbf{63.4$\pm$5.3} \\
    hopp-e & 87.2 & 64.3 & 77.1 & 51.5$\pm$2.9 & \textbf{86.9$\pm$4.2} & \textbf{113.2$\pm$5.9} & 109.5$\pm$2.1 \\
    walk-r & 2.0 & 1.9 & 3.9 & 2.2$\pm$0.1 & \textbf{4.6$\pm$0.7} & 0.0$\pm$0.0 & 0.0$\pm$0.0 \\
    walk-m & 32.6 & 40.3 & 25.0 & 33.0$\pm$2.3 & \textbf{62.3$\pm$6.1} & 40.3$\pm$7.1 & \textbf{61.7$\pm$9.2} \\
    walk-mr & 9.0 & 2.9 & 6.9 & 9.5$\pm$0.4 & \textbf{13.6$\pm$1.2} & 14.1$\pm$1.8 & \textbf{18.8$\pm$1.6} \\
    walk-me & 27.6 & 46.7 & 42.2 & 75.7$\pm$11.8 & \textbf{95.3$\pm$4.6} & 66.7$\pm$5.3 & \textbf{73.4$\pm$6.7} \\
    walk-e & 103.4 & 30.2 & 102.7 & \textbf{108.3$\pm$6.7} & 103.5$\pm$5.9 & 103.5$\pm$1.9 & \textbf{108.8$\pm$3.2} \\
    ant-r & 13.6 & \textbf{31.3} & 26.8 & 14.4$\pm$1.6 & \textbf{16.0$\pm$1.7} & 12.4$\pm$2.2 & \textbf{21.6$\pm$2.0} \\
    ant-m & 89.1 & 36.1 & 96.4 & 91.6$\pm$4.4 & \textbf{101.1$\pm$5.9} & 92.5$\pm$2.7 & \textbf{102.7$\pm$3.4} \\
    ant-mr & 59.7 & 24.0 & 64.1 & 58.2$\pm$7.1 & \textbf{64.8$\pm$4.6} & \textbf{69.6$\pm$8.1} & 57.4$\pm$2.0 \\
    ant-me & 113.1 & 100.5 & 111.9 & 116.8$\pm$3.5 & \textbf{121.2$\pm$3.8} & 107.3$\pm$4.4 & \textbf{120.5$\pm$2.9} \\
    ant-e & 116.3 & 76.3 & 124.5 & 126.8$\pm$1.7 & \textbf{129.0$\pm$2.4} & \textbf{111.0$\pm$2.4} & 107.9$\pm$4.0 \\ 
    \midrule
    \textbf{Total} & 1008.6 & 739.7 & 1019.5 & 1039.0 & \textbf{1198.7} & 1042.1 & \textbf{1156.3} \\ 
    \bottomrule
    \end{tabular}
    \endgroup

\end{table}

\begin{table}[t]
    \centering
    \begingroup
    \setlength{\tabcolsep}{11pt}
    \caption{\textbf{Performance comparison under morphology shifts with extremely limited target domain data.} We report the normalized score evaluated in the target domain, and $\pm$ captures the standard deviation across 5 seeds. We \textbf{bold} the highest scores for each task.}
    \label{tab:limitedappendix}
    \begin{tabular}{l|ccc|cc|cc}
    \toprule
    \textbf{Dataset} & IQL & BOSA & DARA & IGDF & DVDF-IGDF & OTDF & DVDF-OTDF\\
    \midrule
    half-r & 0.0 & \textbf{2.2} & 2.0 & 0.0$\pm$0.0 & 0.0$\pm$0.0 & 2.0$\pm$0.1 & \textbf{2.2$\pm$0.1} \\
    half-m & 18.7 & 17.3 & 16.1 & 22.6$\pm$1.2 & \textbf{26.7$\pm$3.5} & \textbf{24.6$\pm$3.4} & 22.9$\pm$3.6 \\
    half-mr & 12.5 & 9.5 & 8.6 & 14.8$\pm$1.9 & \textbf{19.4$\pm$2.0} & 17.9$\pm$1.6 & \textbf{25.1$\pm$2.4} \\
    half-me & 12.3 & 15.4 & 13.7 & 14.9$\pm$0.5 & \textbf{21.9$\pm$3.1} & 11.5$\pm$0.8 & \textbf{19.8$\pm$1.7} \\
    half-e & 4.9 & 3.6 & 2.9 & \textbf{6.2$\pm$0.1} & 5.9$\pm$0.2 & 10.7$\pm$3.5 & \textbf{15.4$\pm$2.2} \\
    hopp-r & 3.7 & 1.1 & 3.4 & \textbf{4.1$\pm$0.4} & 3.8$\pm$0.1 & \textbf{4.4$\pm$0.2} & 4.0$\pm$0.1 \\
    hopp-m & \textbf{35.2} & 20.6 & 25.5 & \textbf{31.6$\pm$4.2} & 20.3$\pm$2.9 & \textbf{24.2$\pm$3.8} & 19.3$\pm$2.0 \\
    hopp-mr & 2.3 & 3.7 & 3.5 & 4.1$\pm$0.3 & \textbf{7.4$\pm$0.4} & 4.6$\pm$0.2 &\textbf{5.6$\pm$0.3} \\
    hopp-me & \textbf{38.3} & 10.2 & 19.7 & 36.3$\pm$3.7 & \textbf{43.2$\pm$2.8} & 31.6$\pm$2.9 & \textbf{37.4$\pm$3.8} \\
    hopp-e & 28.3 & 7.3 & 13.0 & 29.6$\pm$2.0 & \textbf{44.6$\pm$7.6} & 43.3$\pm$6.2 & \textbf{48.9$\pm$4.1} \\
    walk-r & 0.0 & 0.0 & 0.0 & 0.0$\pm$0.0 & 0.0$\pm$0.0 & 0.0$\pm$0.0 & 0.0$\pm$0.0 \\
    walk-m & 16.4 & 10.6 & 15.8 & 14.3$\pm$2.3 & \textbf{24.3$\pm$1.2} & 19.3$\pm$2.4 & \textbf{23.7$\pm$2.2} \\
    walk-mr & 3.6 & 0.0 & 2.9 & \textbf{4.4$\pm$0.6} & 3.0$\pm$0.2 & 4.1$\pm$0.5 & \textbf{4.8$\pm$0.6} \\
    walk-me & 16.7 & 12.8 & 10.2 & 12.6$\pm$1.0 & \textbf{20.9$\pm$3.7} & 15.4$\pm$1.2 & \textbf{23.0$\pm$1.5} \\
    walk-e & 8.3 & 9.3 & 12.4 & \textbf{13.9$\pm$0.1} & 10.2$\pm$0.4 & 13.5$\pm$0.4 & \textbf{18.9$\pm$0.2} \\
    ant-r & 14.1 & \textbf{20.3} & 16.2 & 13.2$\pm$0.4 & \textbf{17.1$\pm$3.6} & \textbf{10.2$\pm$0.6} & 9.1$\pm$0.2 \\
    ant-m & 17.3 & 30.1 & \textbf{32.9} & 25.6$\pm$2.5 & \textbf{28.1$\pm$5.5} & \textbf{32.3$\pm$4.0} & 26.4$\pm$4.4 \\
    ant-mr & \textbf{29.8} & 19.7 & 13.5 & \textbf{28.7$\pm$1.5} & 19.7$\pm$1.9 & 20.4$\pm$3.0 & \textbf{27.0$\pm$2.2} \\
    ant-me & 15.4 & 15.8 & 12.3 & 17.5$\pm$2.2 & \textbf{21.1$\pm$2.8} & 19.1$\pm$1.8 & \textbf{23.2$\pm$1.9} \\
    ant-e & 20.7 & 20.5 & \textbf{23.1} & 15.8$\pm$1.1 & \textbf{19.0$\pm$3.2} & \textbf{22.7$\pm$1.4} & 16.8$\pm$0.8 \\ 
    \midrule
    \textbf{Total} & 298.5 & 230.0 & 247.7 & 310.2 & \textbf{356.6} & 331.8 & \textbf{373.5} \\ 
    \bottomrule
    \end{tabular}
    \endgroup

\end{table}

\subsection{Extended Results with Extremely Limited Target Data}
\label{appendix:extended}

\begin{table}[t]
    \centering
    \begingroup
    \setlength{\tabcolsep}{5.5pt}
    \caption{\textbf{Performance comparison under kinematic shifts with extremely limited target domain data.} We report the normalized score evaluated in the target domain and $\pm$ captures the standard deviation across 5 seeds. We \textbf{bold} the highest scores for each task.}
    \label{tab:limited}
    \begin{tabular}{l|ccc|cc|cc}
    \toprule
    \textbf{Dataset} & IQL & BOSA & DARA & IGDF & DVDF-IGDF & OTDF & DVDF-OTDF\\
    \midrule
    half-r & 4.8 & 2.2 & \textbf{6.7} & \textbf{5.6$\pm$1.4} & 4.8$\pm$0.5 & \textbf{2.1$\pm$0.1} & 1.7$\pm$0.1 \\
    half-m & 19.8 & 23.6 & 20.4 & 21.6$\pm$0.7 & \textbf{29.7$\pm$1.6} & \textbf{22.8$\pm$1.9} & 21.3$\pm$2.6 \\
    half-mr & 5.9 & 0.0 & 4.0 & \textbf{7.7$\pm$1.2} & 6.6$\pm$2.1 & 4.0$\pm$1.1 & \textbf{9.3$\pm$1.5} \\
    half-me & 9.5 & 11.1 & 7.2 & 14.3$\pm$0.9 & \textbf{22.9$\pm$1.0} & 7.6$\pm$0.4 & \textbf{13.9$\pm$3.7} \\
    half-e & \textbf{7.3} & 4.2 & 6.1 & 4.2$\pm$0.1 & \textbf{6.1$\pm$0.1} & \textbf{5.2$\pm$1.6} & 4.7$\pm$1.0 \\
    hopp-r & 2.4 & 1.9 & 2.2 & 3.7$\pm$0.2 & \textbf{4.2$\pm$0.1} & 1.2$\pm$0.1 & \textbf{5.1$\pm$1.4} \\
    hopp-m & 26.1 & 10.6 & 13.2 & 34.6$\pm$5.9 & \textbf{38.4$\pm$4.1} & \textbf{36.1$\pm$4.4} & 32.8$\pm$3.1 \\
    hopp-mr & 7.4 & 7.8 & 9.8 & 9.8$\pm$1.0 & \textbf{13.0$\pm$2.9} & 14.7$\pm$3.3 & \textbf{21.3$\pm$5.0} \\
    hopp-me & 9.3 & 11.4 & 8.6 & 12.3$\pm$1.4 & \textbf{20.1$\pm$5.2} & 7.1$\pm$2.1 & \textbf{15.2$\pm$2.9} \\
    hopp-e & \textbf{11.1} & 8.3 & \textbf{11.8} & \textbf{9.4$\pm$0.5} & 8.0$\pm$0.2 & \textbf{6.7$\pm$0.3} & 6.4$\pm$0.1 \\
    walk-r & 4.6 & 0.0 & 0.0 & \textbf{8.1$\pm$2.9} & 6.6$\pm$1.4 & 0.0$\pm$0.0 & 0.0$\pm$0.0 \\
    walk-m & 7.7 & 7.6 & 4.4 & 14.3$\pm$1.7 & \textbf{23.0$\pm$3.9} & 11.8$\pm$1.9 & \textbf{16.1$\pm$3.7} \\
    walk-mr & 3.9 & 9.1 & 4.3 & 2.4$\pm$0.1 & \textbf{3.7$\pm$0.1} & 7.4$\pm$1.3 & \textbf{16.0$\pm$1.6} \\
    walk-me & 5.7 & 4.8 & 6.4 & 8.4$\pm$2.1 & \textbf{16.2$\pm$4.4} & 8.1$\pm$2.4 & \textbf{15.9$\pm$3.8} \\
    walk-e & 10.6 & 9.3 & 20.1 & \textbf{13.7$\pm$2.8} & 11.9$\pm$1.6 & 15.8$\pm$2.0 & \textbf{19.3$\pm$1.2} \\
    ant-r & 7.0 & 6.5 & 5.5 & \textbf{11.8$\pm$3.0} & 8.4$\pm$2.2 & 7.3$\pm$0.5 & \textbf{10.1$\pm$0.6} \\
    ant-m & 14.6 & 19.1 & 21.3 & 20.3$\pm$1.2 & \textbf{24.1$\pm$3.3} & 42.3$\pm$7.7 & \textbf{48.1$\pm$6.3} \\
    ant-mr & 7.3 & \textbf{17.6} & 13.2 & 3.9$\pm$0.7 & \textbf{13.4$\pm$2.6} & \textbf{17.6$\pm$2.8} & 14.7$\pm$2.3 \\
    ant-me & 5.3 & 10.1 & 2.9 & \textbf{9.4$\pm$4.6} & 9.0$\pm$4.8 & 4.3$\pm$0.5 & \textbf{12.3$\pm$4.6} \\
    ant-e & 3.1 & 4.3 & 0.0 & 2.9$\pm$1.4 & \textbf{5.5$\pm$0.8} & \textbf{5.1$\pm$1.2} & 4.7$\pm$1.5 \\ 
    \midrule
    \textbf{Total} & 173.4 & 169.5 & 141.3 & 218.4 & \textbf{275.6} & 227.2 & \textbf{288.9} \\ 
    \bottomrule
    \end{tabular}
    \endgroup

\end{table}

In this section, we consider a more challenging setting compared with Section~\ref{sec:main} following~\citep{lyu2025cross, lyu2024odrlabenchmark}, where only extremely limited target domain data (around 5,000 transitions) are available. This setting reflects real-world scenarios, such as nuclear power plant control, where accessing more target domain data is often impractical. Typical offline RL will fail under such extreme data scarcity, making the proper utilization of source domain data much more crucial.

\textbf{Tasks and Datasets.} The tasks and types of dynamics shifts are identical to those in Section~\ref{sec:main}. The only distinction lies in the target domain datasets, which now consist of only 5,000 transitions sampled from the D4RL datasets, instead of the $10\%$ subset used in Section~\ref{sec:main}.

\textbf{Baselines.} We maintain the same baselines (IQL, BOSA, DARA, IGDF, and OTDF), and implement DVDF-IGDF and DVDF-OTDF for comparison as in Section~\ref{sec:main}.

\textbf{Experimental Results.} We run each algorithm for 1M gradient steps with 5 random seeds. We present the empirical results under kinematic shifts in Table~\ref{tab:limited}, and the results under morphology shifts in Table~\ref{tab:limitedappendix}. 

As shown in Table~\ref{tab:limited}, DVDF substantially enhances the performance of base algorithms, elevating total normalized scores by $\textbf{26.2\%}$ (IGDF) and $\textbf{27.1\%}$ (OTDF) under kinematic shifts. Specifically, DVDF-IGDF outperforms IGDF in \textbf{13} out of 20 tasks, and DVDF-OTDF surpasses OTDF in \textbf{12} out of 20 tasks, while achieving comparable performance in the remaining tasks. 

The results in Table~\ref{tab:limitedappendix} demonstrate that DVDF maintains superiority over baseline methods under morphology shifts: DVDF-OTDF achieves the highest total score of \textbf{373.5} among all methods, surpassing OTDF by $\textbf{12.6\%}$ and leading in \textbf{13} out of 20 tasks. Similarly, DVDF-IGDF improves IGDF by $\textbf{15.0\%}$ in total score and achieves better performance in \textbf{12} out of 20 tasks. These results demonstrate the superiority of DVDF with extremely limited target domain data.

\subsection{More Comparisons with Recent Studies}

{In this section, we compare our method DVDF with two more recent studies, PSEC~\citep{liuskill} and DmC~\citep{le2025dmc}. PSEC proposes to preserve the prior learned skills in a parametric space and adaptively composes them using a context-aware module to handle new tasks. In the cross-domain offline setting, PSEC first learns separate policies from the source domain and target domain data, which are then dynamically combined to work under the target dynamics. DmC employs a KNN-based estimator as a measure of the dynamics gap, and utilizes the KNN proximity score as a guiding signal for diffusion-based data generation. Source domain samples are selected based on the proximity score and combined with the target data for training. As a plug-in method, DVDF could be seamlessly integrated into both PSEC and DmC. Specifically, DVDF could assist PSEC by selecting beneficial source domain samples to facilitate target policy learning from the limited target dataset. Similarly, DVDF could be integrated into DmC's source data selection process to enable dynamics- and value-aligned data selection. We refer to these two integrated methods as DVDF-PSEC and DVDF-DmC. }

{We evaluate DVDF-PSEC and DVDF-DmC on four tasks (\texttt{halfcheetah}, \texttt{hopper}, \texttt{walker2d}, \texttt{ant}) under kinematic shifts, with datasets of three qualities: \texttt{medium}, \texttt{medium-expert}, and \texttt{expert}. All other experimental settings follow Section~\ref{sec:main}. The performance comparison between the base algorithms (PSEC, DmC) and their DVDF-enhanced versions (DVDF-PSEC and DVDF-DmC) is presented in Table~\ref{tab:psec_dmc}.}

{The results in Table~\ref{tab:psec_dmc} show that DVDF-PSEC outperforms PSEC on 11 out of 12 datasets, while DVDF-DmC surpasses DmC on 10 out of 12. Furthermore, the DVDF-enhanced versions achieve a substantially higher total score, confirming the versatility of DVDF as a plug-in module.}

\begin{table}[t]
    \centering
    \caption{\textbf{Performance comparison with PSEC and DmC under kinematic shifts.} We report the normalized score evaluated in the target domain, and $\pm$ captures the standard deviation across 5 seeds. We \textbf{bold} the highest scores for each task.}
    \label{tab:psec_dmc}
    \begin{tabular}{l|cc|cc}
    \toprule
    \textbf{Dataset} & PSEC & DVDF-PSEC & DmC & DVDF-DmC \\
    \midrule
    half-kine-m & \textbf{33.4$\pm$1.0} & 30.1$\pm$0.8 & \textbf{39.6$\pm$1.1} & 33.2$\pm$0.2 \\
    half-kine-me & 41.6$\pm$2.3 & \textbf{49.3$\pm$2.7} & 47.3$\pm$5.8 & \textbf{52.6$\pm$3.2} \\
    half-kine-e & 52.0$\pm$4.7 & \textbf{58.3$\pm$3.4} & 66.9$\pm$2.4 & \textbf{73.6$\pm$1.6} \\
    hopper-kine-m & 47.2$\pm$3.5 & \textbf{53.2$\pm$1.5} & 53.8$\pm$5.0 & \textbf{65.4$\pm$2.8} \\
    hopper-kine-me & 31.7$\pm$3.6 & \textbf{42.0$\pm$2.4} & 47.2$\pm$1.3 & \textbf{56.3$\pm$2.8} \\
    hopper-kine-e & 70.1$\pm$4.3 & \textbf{74.2$\pm$3.4} & 92.6$\pm$2.7 & \textbf{98.2$\pm$1.6} \\
    walker2d-kine-m & 41.0$\pm$1.9 & \textbf{56.7$\pm$5.2} & 48.3$\pm$4.8 & \textbf{65.1$\pm$2.4} \\
    walker2d-kine-me & 53.7$\pm$3.0 & \textbf{58.6$\pm$3.6} & 62.6$\pm$2.9 & \textbf{69.2$\pm$1.2} \\
    walker2d-kine-e & 75.9$\pm$1.4 & \textbf{95.8$\pm$2.8} & 83.0$\pm$2.6 & \textbf{93.0$\pm$2.6} \\
    ant-kine-m & 82.2$\pm$1.5 & \textbf{87.4$\pm$1.4} & \textbf{84.1$\pm$5.5} & 81.2$\pm$2.0 \\
    ant-kine-me & 94.6$\pm$1.7 & \textbf{106.0$\pm$2.9} & 107.1$\pm$3.5 & \textbf{112.3$\pm$2.1} \\
    ant-kine-e & 103.5$\pm$1.0 & \textbf{118.4$\pm$1.5} & 101.2$\pm$1.0 & \textbf{116.1$\pm$2.9} \\
    \midrule
    \textbf{Total} & 726.9 & \textbf{830.0} & 833.7 & \textbf{916.2} \\
    \bottomrule
    \end{tabular}
\end{table}

\section{Compute Infrastructure}
\label{appendix:compute}
We list our compute infrastructure for our experiments in Table~\ref{tab:compute}.

\begin{table}[H]
    \centering
    \caption{Compute Infrastructure}
    \label{tab:compute}
    \begin{tabular}{c|c|c}
    \toprule
    \textbf{CPU} & \textbf{GPU} & \textbf{Memory}\\
    \midrule
    AMD EPYC 7452 & RTX3090$\times$8 & 288GB\\
    \bottomrule
    \end{tabular}

\end{table}

\section{Training Time}

We report the average training time for our method and the baselines (including IQL, IGDF, OTDF, DVDF-IGDF, and DVDF-OTDF) across various tasks over 1M steps in Table~\ref{tab:time}. Note that the additional computational overhead for OTDF and DVDF stems from solving complex optimal transport matrices and pre-training for the advantage function, respectively. Fortunately, these computations can be precomputed, minimizing their impact on subsequent experiments.

\begin{table}[H]
    \centering
    \caption{Training time comparison between various methods. h=hour(s), m=minute(s).}
    \label{tab:time}
    \begin{tabular}{c|c|c|c|c}
    \toprule
    IQL & IGDF & OTDF & DVDF-IGDF & DVDF-OTDF \\
    \midrule
    5h24m & 6h56m & 9h17m & 11h43m & 14h07m \\
    \bottomrule
    \end{tabular}

\end{table}

\section{Broader Impacts}
\label{appendix:impacts}
This paper presents work whose goal is to promote effective cross-domain offline RL. Our work has potential positive social impacts. For example, our research could enable more efficient development of advanced robotics systems by effectively utilizing source domain data. At present, we have not identified any foreseeable negative impacts arising from this research.

\section{LLM Usage Declaration}

The use of LLMs in this work is strictly limited to grammatical polishing of the initial draft. LLMs are not involved in any core research components, including but not limited to the conception of the method, theoretical proofs, and experiments.

\end{document}